%% file: lightning2.tex
\documentclass{article}

\input{math_commands.tex}

\usepackage{hyperref}
\usepackage{url}
\usepackage[utf8]{inputenc} 
\usepackage[T1]{fontenc}    
\usepackage{hyperref}       
\usepackage{url}            
\usepackage{booktabs}       
\usepackage{amsfonts}       
\usepackage{nicefrac}       
\usepackage{microtype}      
\usepackage{xcolor}         
\usepackage{makecell}
\usepackage{graphicx}
\usepackage{amsmath}
\usepackage{multirow}
\usepackage{wrapfig}
\usepackage{wrapfig}
\usepackage{float}
\usepackage{graphicx}
\usepackage{caption}
\usepackage{color}
\usepackage{colortbl}

\def\ie{\emph{i.e., }}

\newcommand{\rnum}[1]{\uppercase\expandafter{\romannumeral #1\relax}}

\usepackage{microtype}
\usepackage{graphicx}
\usepackage{subfigure}
\usepackage{booktabs} 
\usepackage{paralist}

\usepackage{hyperref}

\usepackage[accepted]{icml2024}

\usepackage{amsmath}
\usepackage{amssymb}
\usepackage{mathtools}
\usepackage{amsthm}

\usepackage[capitalize,noabbrev]{cleveref}

\theoremstyle{plain}

\theoremstyle{definition}

\theoremstyle{remark}

\usepackage[textsize=tiny]{todonotes}

\expandafter\def\expandafter\normalsize\expandafter{%
    \normalsize%
    \setlength\abovedisplayskip{2pt}%
    \setlength\belowdisplayskip{4pt}%
    \setlength\abovedisplayshortskip{-4pt}%
    \setlength\belowdisplayshortskip{2pt}%
}

\icmltitlerunning{TNL with Lightning Attention}

\begin{document}

\twocolumn[
\icmltitle{Various Lengths, Constant Speed: Efficient Language Modeling with Lightning Attention}



\begin{icmlauthorlist}
\icmlauthor{Zhen Qin}{tap}
\icmlauthor{Weigao Sun}{lab}
\icmlauthor{Dong Li}{lab}
\icmlauthor{Xuyang Shen}{lab}
\icmlauthor{Weixuan Sun}{lab}
\icmlauthor{Yiran Zhong}{lab}
\end{icmlauthorlist}

\icmlaffiliation{lab}{OpenNLPLab, Shanghai AI Lab}
\icmlaffiliation{tap}{TapTap}

\icmlcorrespondingauthor{Yiran Zhong}{zhongyiran@gmail.com}

\icmlkeywords{Linear attention,Lightning attention,unlimited sequence length, large language model}
\vskip 0.3in
]


\printAffiliationsAndNotice{}  

\begin{abstract}
We present Lightning Attention, the first linear attention implementation that maintains a constant training speed for various sequence lengths under fixed memory consumption. Due to the issue with cumulative summation operations (\texttt{cumsum}), previous linear attention implementations cannot achieve their theoretical advantage in a casual setting. However, this issue can be effectively solved by utilizing different attention calculation strategies to compute the different parts of attention. Specifically, we split the attention calculation into intra-blocks and inter-blocks and use conventional attention computation for intra-blocks and linear attention kernel tricks for inter-blocks. This eliminates the need for \texttt{cumsum} in the linear attention calculation. Furthermore, a tiling technique is adopted through both forward and backward procedures to take full advantage of the GPU hardware. To enhance accuracy while preserving efficacy, we introduce TransNormerLLM (TNL), a new architecture that is tailored to our lightning attention. We conduct rigorous testing on standard and self-collected datasets with varying model sizes and sequence lengths. TNL is notably more efficient than other language models. In addition, benchmark results indicate that TNL performs on par with state-of-the-art LLMs utilizing conventional transformer structures. The source code is released at github.com/OpenNLPLab/TransnormerLLM.
\end{abstract}

\section{Introduction}
\label{sec: intro}

\begin{figure*}[t]
    \centering
\includegraphics[width=1\textwidth]{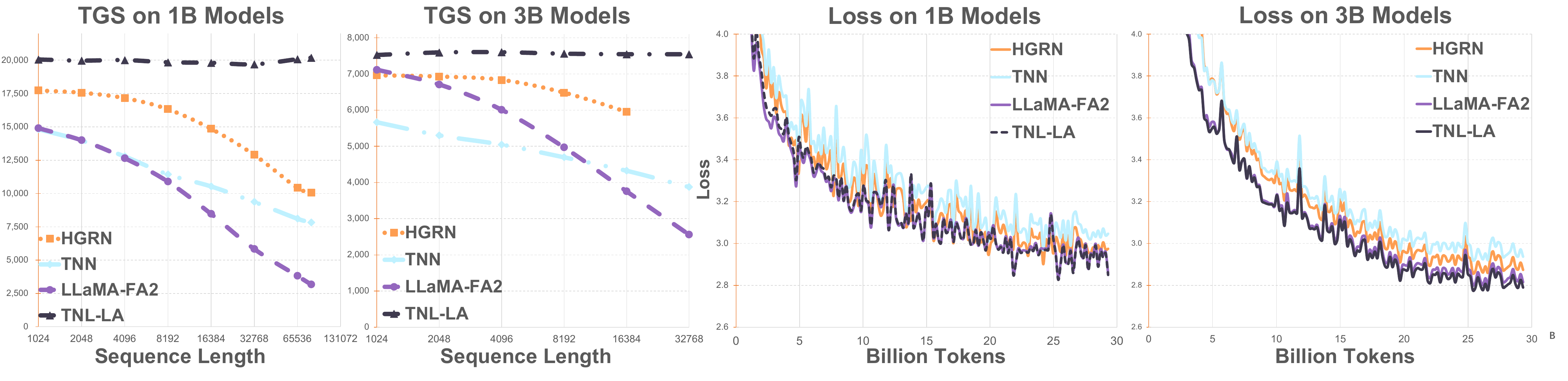}
    \vspace{-6mm}
    \caption{\textbf{Training speed and accuracy comparison.} We compare TNL's training speed and losses with state-of-the-art transformer models (LLaMA with FlashAttention-2) and efficient non-transformer models (HGRN~\cite{qin2023hierarchically} and TNN~\cite{qin2023toeplitz}). TNL achieves the lowest training losses and maintains consistent training speed regardless of sequence length.}
    \vspace{-2mm}
    \label{fig: tgs}
\end{figure*}

Linear attention has emerged as a potentially viable alternative to conventional softmax attention over the last five years~\cite{bahdanau2016neural,debrébisson2016cheap}. However, despite its promise, none of the current leading large language models~\cite{touvron2023llama,2307.09288,zeng2022glm,gpt-neo,falcon40b,mpt-7b,wang2021gpt,baichuan2023baichuan2,jiang2023mistral} have adopted linear attention mechanisms.
There are two possible reasons for that:
1).~\emph{Inferior performance:} There is a notable performance gap between existing linear attention-based models~\cite{katharopoulos2020transformers,zhen2022cosformer} and state-of-the-art softmax attention-based models~\cite{touvron2023llama,2307.09288} in language modeling.
2).~\emph{Slow training speed:} Existing linear attention models frequently struggle with slow training speeds due to the use of cumulative summation operations (\texttt{cumsum})~\cite{hua2022transformer}. As a result, these models~\cite{hua2022transformer} often adopt conventional attention computation during practical use, losing the theoretical advantages of linear attention.

In this paper, we address the aforementioned issues of linear attention and propose a new linear attention-based model that outperforms softmax attention-based models in terms of accuracy and efficiency in language modeling.

\textbf{\emph{Training speed.}} We introduce Lightning Attention, the first linear attention implementation that enables linear attention to realize its theoretical computational benefits. To achieve the linear computational complexities, the core idea is to leverage the "kernel trick" to accelerate the attention matrix computation, \ie~compute the product of keys and values first to circumvent the $n \times n$ query-key matrix multiplication. The slow operation \texttt{cumsum} is needed during the calculation in causal language modeling.
To solve this dilemma, we apply the concept of "divide and conquer" to perform the calculation. Specifically, our attention calculation is divided into intra-blocks and inter-blocks. The conventional attention calculation is applied to intra-blocks, while the "kernel trick" is utilized for inter-blocks. We also leverage tiling techniques in both forward and backward processes to maximize GPU hardware performance and tailor the technique used in FlashAttention~\cite{dao2022flashattention,dao2023flashattention2} to our Lightning Attention to make it IO-friendly. As a result, Lightning Attention maintains a constant training speed with increasing sequence length under fixed memory consumption, as shown in Fig.~\ref{fig: tgs}.

\textbf{\emph{Accuracy.}} As the adage goes, a good horse often needs a good spur. We propose a novel architecture, TransNormerLLM (TNL), which is specifically designed for Lightning Attention in order to enhance its performance. TNL evolves from the previous linear attention architecture TransNormer~\citep{qin-etal-2022-devil} by making advanced modifications that include positional embedding, linear attention acceleration, gating mechanism, tensor normalization. Specifically, we use LRPE~\citep{qin2023linearized} together with an exponential decay to avoid attention dilution issues while allowing the model to retain global interactions between tokens. A gating mechanism is utilized to smooth training, and a new tensor normalization scheme is proposed to accelerate the model while preserving its accuracy. We also implement an efficient model parallel schema for TransNormerLLM, enabling seamless deployment on large-scale clusters and facilitating expansion to even more extensive models. As shown in Fig.~\ref{fig: tgs}, TNL achieves the lowest training loss among the existing efficient transformer structures~\cite{qin2023toeplitz,qin2023hierarchically} as well as SOTA transformer models~\cite{2307.09288}.

We perform a comprehensive evaluation of Lightning Attention across a diverse range of sequence lengths to assess its accuracy and compare its computational speed and memory utilization with FlashAttention-2 ~\cite{dao2023flashattention2}. Lightning Attention exhibits a notable advantage in computational speed and memory consumption compared to its counterparts without compromising performance. We also validate our model design through a series of ablations and train models with sizes of 44M, 385M, 1B, 7B, and 15B on standard or our self-collected datasets. Benchmark results demonstrate that TNL not only matches the performance of SOTA LLMs with Transformer but is also significantly faster.

\section{Related Work}
\label{sec: preliminary}
\subsection{Efficient Language Modeling}
New efficient model architectures are being explored to address the high time complexity of the traditional transformer structure. Four promising alternatives, including linear transformers, state space models, long convolution, and linear recurrence, are being developed to replace self-attention modules for long sequence modeling.

\textbf{Linear Attention}
Linear attention decomposes Softmax Attention into the inner product of hidden representations, allowing it to use the "Kernel Trick", where the product of keys and values is computed first to avoid the quadratic $n \times n$ matrix. Different methods utilize various hidden representations. For example, ~\citet{katharopoulos2020transformers} use 1+elu as an activation function, ~\citet{zhen2022cosformer} use the cosine function to approximate the properties of softmax, and ~\citet{choromanski2021rethinking,zheng2022linear,zheng2023efficient} approximate softmax through theoretical approaches. Although its theoretical complexity is $O(nd^2)$, the actual computational efficiency of linear attention becomes low when used in causal attention due to the need for \textit{cumsum} operations~\citep{hua2022transformer}. Moreover, most linear attention still exhibits a certain performance gap compared to traditional Transformers~\citep{katharopoulos2020transformers,liu2022neural}.

\textbf{State Space Model}
State Space Model is based on the State Space Equation for sequence modeling~\citep{s4}, using special initialization~\citep{2008.07669,gu2022parameterization}, diagonalization assumptions~\citep{gupta2022diagonal}, and mixed techniques~\citep{h3} to achieve performance comparable to Transformers. Due to the characteristics of the state space equation, inference can be conducted with constant complexity~\citep{s4}, whereas the training speed can be slow compared with FlashAttention.

\textbf{Long Convolution}
Long convolution models~\citep{qin2023toeplitz,simplelongconv} utilize a kernel size equal to the input sequence length, facilitating a wider context compared to traditional convolutions. Training these models involves Fast Fourier Transforms (FFT) algorithm, reducing the computational complexities to $O(n\log n)$. However, long convolution models need to cache all historical computations for causal convolution inference, making them less ideal for processing long sequences compared to RNNs.

\textbf{Linear RNN}
Linear RNNs~\citep{2303.06349,qin2023hierarchically}, in contrast, stand out as more suitable replacements for transformers in long-sequence modeling. A notable example is the HGRN~\citep{qin2023hierarchically} model, a linear RNN-based LLM that has shown competitive performance against similarly scaled GPT models.

\subsection{IO-aware Attention}
The FlashAttention series~\citep{dao2022flashattention,dao2023flashattention2} focuses on system-level optimizations for the efficient implementation of the standard attention operator on GPU platforms. These approaches employ tiling strategies to minimize the volume of memory reads/writes between the GPU's high bandwidth memory (HBM) and on-chip SRAM. Although these methods optimize the IO communication in attention calculation and are faster than previous softmax attention implementations, their theoretical computation complexity remains $O(n^2d)$, making them unsuitable for long sequence language modeling.

\setlength{\textfloatsep}{2pt}
\begin{algorithm}[t]
\small
    \caption{Linear Attention Left Product}
    \label{algo:left product}
    \begin{algorithmic}
    \STATE{\textbf{Input:} $\mathbf Q,\mathbf K,\mathbf V \in \mathbb{R}^{n \times d}$.}
     \STATE{Initialize mask $\mathbf M\in \mathbb R^{n\times n}$, where $\mathbf M_{ts} = 1$, if $t\ge s$, else 0.}
        \STATE{Load $\mathbf Q,\mathbf  K, \mathbf M $ from HBM, compute $\mathbf {S}=(\mathbf Q\mathbf K^{\top})\odot \mathbf M$, write $\mathbf {S}$ to HBM.}
  \STATE{Load $\mathbf S, \mathbf V$ from HBM, compute $\mathbf {O}=\mathbf S \mathbf V$, write $\mathbf {O}$ to HBM.}
    \STATE{Return $\mathbf O$.}
\end{algorithmic}
\end{algorithm}

\section{Lightning Attention}
\label{sec: algo}
\subsection{Preliminary}
We first recall the formulation of linear attention and then introduce our proposed Lightning Attention.
In the case of NormAttention within TransNormer~\citep{qin-etal-2022-devil}, attention computation deviates from the conventional Transformer structure~\cite{vaswani2017attention} by eschewing the costly softmax and scaling operations. The NormAttention mechanism can be expressed as follows:
\begin{equation}
\mathbf{O}=\mathrm{Norm}((\mathbf{Q} \mathbf{K}^{\top})\mathbf{V}),
\label{eq: norm attention}
\end{equation}
where $\mathbf{Q}$, $\mathbf{K}$, and $\mathbf{V} \in \R^{n\times d}$ are the query, key, and value matrices, respectively, with $n$ for sequence length and $d$ for feature dimension. The equation can be transformed into its linear variant using right matrix multiplication:
\begin{equation}
\small
\mathbf{O}=\mathrm{Norm}(\mathbf{Q} (\mathbf{K}^{\top}\mathbf{V})),
\label{eq: norm attention 2}
\end{equation}
The linear formulation enables efficient recurrent prediction with $O(nd^2)$ complexity during training. Additionally, linear attention guarantees a constant computation complexity of $O(d^2)$ regardless of the sequence length. This is achieved by recurrently updating $\mathbf{K}^{\top}\mathbf{V}$, eliminating the need for repeated computation of the entire attention matrix. In contrast, standard softmax attention has a complexity of $O(nd^2)$ during inference.

Nevertheless, when dealing with causal prediction tasks, the effectiveness of the right product is compromised, leading to the requirement for the computation of \texttt{cumsum}~\citep{hua2022transformer}. This impediment hinders the potential for highly efficient parallel computation. In this section, we show that the requirement of \texttt{cumsum} can be eliminated by leveraging the concept of "divide and conquer" in linear attention calculation. For the convenience of discussion, Norm will be ignored in the subsequent discussion.

There are two computational approaches to handling the causal scenario.
One is using conventional attention computation (the Left Product), which involves computing $\mathbf Q\mathbf K^\top$ first.
The complete calculation formula is as follows:
\begin{equation}
\label{eq:rnn}
\mathbf O=[(\mathbf Q\mathbf K^\top)\odot \mathbf M] \mathbf V
\end{equation}
where $\mathbf M_{ts}=1$ if $t\ge s$, otherwise 0. The complete algorithm is detailed in Algorithm~\ref{algo:left product}. Note that this algorithm is parallelizable, but its time complexity is $O(n^2d)$.
The other option is to compute the $\mathbf k_t\mathbf v_t^{\top}$ first (the Right Product), which leverages a recursive formula for computation:
\begin{equation}
\label{eq:fwd}
\mathbf {kv}_0=\mathbf 0, \mathbf {kv}_t=\mathbf {kv}_{t-1} + \mathbf k_t\mathbf v_t^\top, \mathbf o_t^{\top} = \mathbf q_t^{\top} \mathbf {kv_t}.
\end{equation}
The complete algorithm is detailed in Algorithm~\ref{algo:right product}. This algorithm has a time complexity of $O(nd^2)$, but it is not GPU-friendly, making it slower than the first approach.

\begin{algorithm}[t]
\small
    \caption{Linear Attention Right Product}
    \label{algo:right product}
    \begin{algorithmic}
    \STATE{\textbf{Input:} $\mathbf Q,\mathbf K,\mathbf V \in \mathbb{R}^{n \times d}$.}
     \STATE{Initialize $\mathbf {kv} =0\in \mathbb R^{d\times d} $.}
    \FOR{$t=  1,\ldots ,n$}
        \STATE{Load $\mathbf q_t,\mathbf  k_t, \mathbf v_t \in \mathbb{R}^{d \times 1}$ from HBM to on-chip SRAM.}
\STATE{On chip, compute $\mathbf{kv} = \mathbf{kv}+ \mathbf k_t  \mathbf v_t^{\top}$.}
        
        \STATE{On chip, compute $\mathbf{o}_{t} = \mathbf q_t^{\top}\mathbf {kv} $.}
      \STATE{Write $\mathbf o_t^{\top}$ to HBM as the $t$-th row of $\mathbf O$.}
      \ENDFOR
      \STATE{Return $\mathbf O$.}
\end{algorithmic}
\end{algorithm}

\subsection{Linear Attention with Tiling}
We use a tiling technique to compute linear attention in a causal setting. Specifically, we first divide $\mathbf Q, \mathbf K, \mathbf V$ into two blocks by rows:
\begin{equation*}
\begin{gathered}
\textstyle
\mathbf X=\left[\begin{matrix}
\mathbf X_1\\
\mathbf X_2
\end{matrix}\right], \mathbf X_1 \in \mathbb R^{m\times d}, \mathbf X_2 \in \mathbb R^{(n - m)\times d}, \\
\mathbf X\in \{\mathbf Q, \mathbf K, \mathbf V\}.
\end{gathered}
\end{equation*}
\normalsize
Then, by unfolding Eq.~\ref{eq:rnn}, we get (note that $\mathbf {kv}_0=0)$:
\begin{equation}
\begin{aligned}
\textstyle
\mathbf {kv}_{s}&=\mathbf {kv}_{0}+\sum_{j=1}^{s} \mathbf k_j\mathbf v_j^{\top},s=1,\ldots, m.\\ 
\mathbf o_{s}^{\top} &=\mathbf q_{s}^{\top} \mathbf {kv}_{s}=\mathbf q_{s}^{\top}\mathbf {kv}_{0} + \mathbf q_{s}^{\top}\sum_{j=1}^{s} \mathbf k_j\mathbf v_j^{\top}. \\
\end{aligned}
\end{equation}
In block form, we have:
\begin{equation}
\begin{aligned}
\mathbf O_1 
& = \mathbf Q_1 \mathbf {kv}_0 + [(\mathbf Q_1 \mathbf K_1^{\top})\odot \mathbf M]\mathbf V_1 \\ 
&\triangleq
\mathbf Q_1 \mathbf {KV}_0 + [(\mathbf Q_1 \mathbf K_1^{\top})\odot \mathbf M]\mathbf V_1.
\end{aligned}
\end{equation}
The above formula shows that the forward causal linear attention can be divided into two parts:
\begin{compactitem}
\item The computation within the block $ [(\mathbf Q_1 \mathbf K_1^{\top})\odot \mathbf M]\mathbf V_1$ (intra blocks) can use the Left Product;
\item The computation between blocks $\mathbf Q_1 \mathbf {KV}_0$ (inter blocks) can use the Right Product.
\end{compactitem}

It is worth noting that the second block can be computed using the same idea as follows:
\begin{equation}
\begin{aligned}
 \mathbf {kv}_{m+t}&=\mathbf  {kv}_{m}+\sum_{j=m+1}^{m+t} \mathbf k_j\mathbf v_j^{\top},t=1,\ldots,n-m, \\
 \mathbf o_{m+t}^{\top} &=\mathbf q_{m+t}^{\top} \mathbf {kv}_{m+t},\\
 \mathbf O_2 
& = \mathbf Q_2 \mathbf {kv}_m + [(\mathbf Q_2 \mathbf K_2^{\top})\odot \mathbf M]\mathbf V_2 \\
&\triangleq
\mathbf Q_2 \mathbf {KV}_1 + [(\mathbf Q_2 \mathbf K_2^{\top})\odot \mathbf M]\mathbf V_2.
\end{aligned}
\end{equation}
Note that to compute the second block, we have to use $\mathbf{KV}_1=\mathbf {kv}_m$, which can be computed by:
\begin{equation}
\mathbf{KV}_1 = \mathbf {KV}_0+\sum_{j=1}^{m} \mathbf k_m\mathbf v_m^{\top}= \mathbf {KV}_0+\mathbf K_1^{\top}\mathbf V_1.
\end{equation}
where $\mathbf {KV}_0=\mathbf {kv}_0$. By using the above strategy to divide the matrix into multiple blocks, we obtain the Lightning Attention Forward Pass. More detailed derivation can be found in the Appendix~\ref{proof}.

\begin{algorithm}[t]
\small
    \caption{Lightning Attention Forward Pass}
    \label{algo:Lightning Attention fw pseudo}
    \begin{algorithmic}
    \STATE{\textbf{Input:} $\mathbf Q,\mathbf K,\mathbf V \in \mathbb{R}^{n \times d}$, block sizes $B$.}
    \STATE{Divide $\mathbf {X}$ into $T = \frac{n}{B}$ blocks $\mathbf X_1, \mathbf X_2, ...\mathbf X_{T}$ of size $B \times d$ each, where $\mathbf X\in \{\mathbf Q, \mathbf K, \mathbf V,\mathbf O \}$. }
     \STATE{Initialize mask $\mathbf M\in \mathbb R^{B\times B}$, where $\mathbf M_{ts} = 1$, if $t\ge s$, else 0.} 
     \STATE{Initialize $\mathbf {KV} =0\in \mathbb R^{d\times d} $.}
    \FOR{$t=  1,\ldots ,T$}
        \STATE{Load $\mathbf Q_t,\mathbf  K_t, \mathbf V_t \in \mathbb{R}^{B \times d}$ from HBM to on-chip SRAM.}
        \STATE{On chip, compute $\mathbf O_{\mathrm{intra}}= [(\mathbf Q_t \mathbf K_t^{\top }) \odot \mathbf M]\mathbf V_t$.}
        \STATE{On chip, compute $\mathbf{O}_{\mathrm{inter}} = \mathbf Q_t (\mathbf {KV}) $.}
        \STATE{On chip, compute $\mathbf{KV} = \mathbf{KV}+ \mathbf K_t^{\top}  \mathbf V_t$.}
      \STATE{Write $\mathbf O_t=\mathbf O_{\mathrm{intra}}+ \mathbf{O}_{\mathrm{inter}}$ to HBM as the $t$-th block of $\mathbf O$.}
      \ENDFOR
      \STATE{Return $\mathbf O$.}
\end{algorithmic}
\end{algorithm}

\begin{figure}[t]
  \centering
  \includegraphics[width=1\columnwidth]{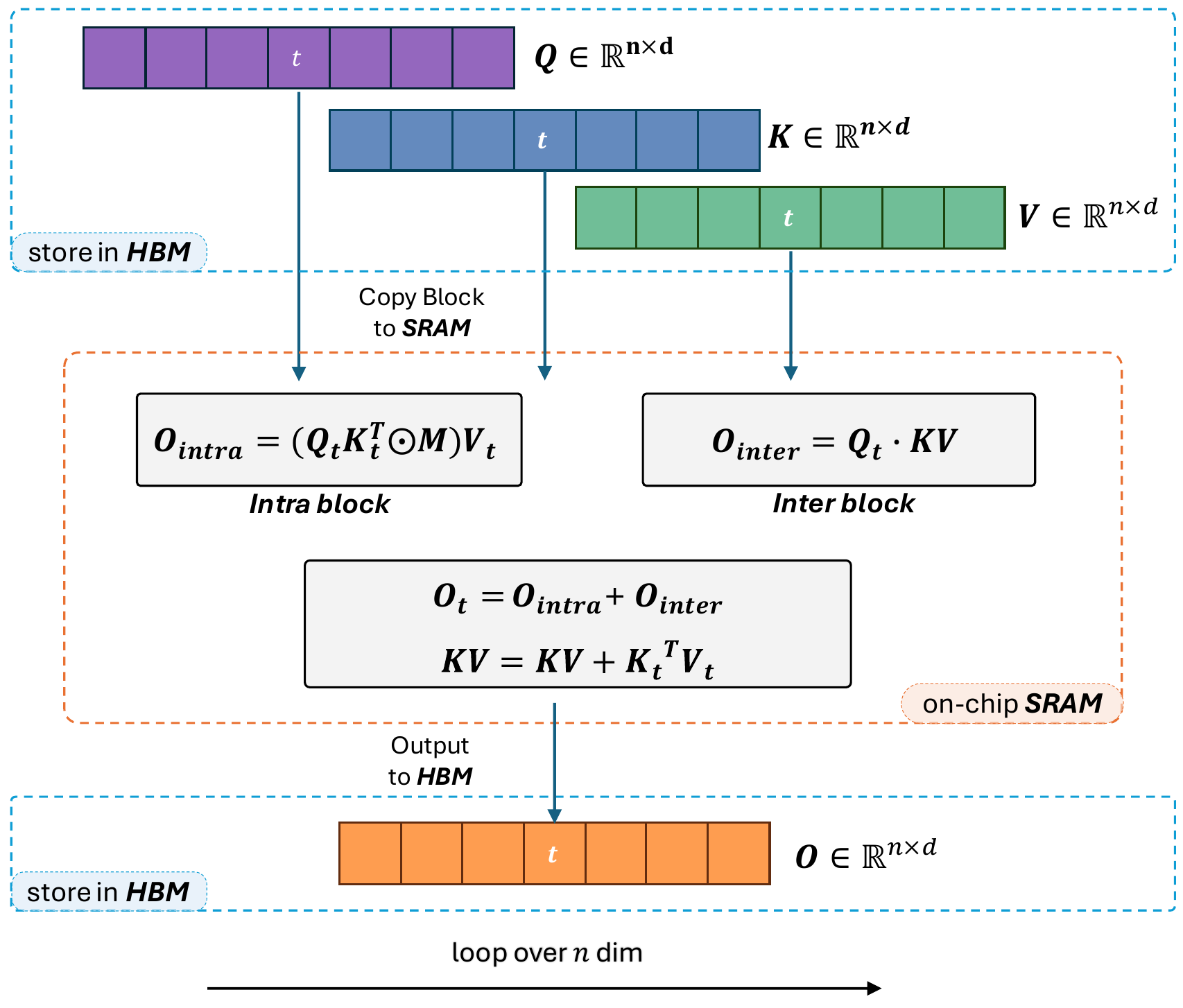}
  \vspace{-4mm}
  \caption{\textbf{Structural framework of Lightning Attention} is detailed in its algorithmic schematic. During the $t$-th iteration, the tiling blocks of matrices $\mathbf{Q}_t, \mathbf{K}_t, \mathbf{V}_t$ are transferred from High Bandwidth Memory (HBM) to Static Random-Access Memory (SRAM). Within the SRAM, the outputs $\mathbf{O}_{\mathrm{intra}}$ and $\mathbf{O}_{\mathrm{inter}}$ are computed independently, followed by an update to the $\mathbf{KV}$ matrix. Subsequently, the final output $\mathbf{O}_t$, which is the sum of $\mathbf{O}_{\mathrm{intra}}$ and $\mathbf{O}_{\mathrm{inter}}$, is written back from SRAM to HBM.}
  \vspace{5mm}
  \label{fig:lightning2}
\end{figure}

\begin{algorithm}[t]
\small
    \caption{Lightning Attention Backward Pass}
    \label{algo:Lightning Attention bw pseudo}
    \begin{algorithmic}
    \STATE{\textbf{Input:} $\mathbf Q,\mathbf K,\mathbf V,\mathbf{dO} \in \mathbb{R}^{n \times d}$, block sizes $B$.}
    \STATE{Divide $\mathbf {X}$ into $T = \frac{n}{B}$ blocks $\mathbf X_1, \mathbf X_2, ...\mathbf X_{T}$ of size $B \times d$ each, where $\mathbf X\in \{\mathbf Q, \mathbf K, \mathbf V \}$. }
     \STATE{Divide $\mathbf {dX}$ into $T = \frac{n}{B}$ blocks $\mathbf {dX}_1, \mathbf {dX}_2, ...\mathbf {dX}_{T}$ of size $B \times d$ each, where $\mathbf X\in \{\mathbf Q, \mathbf K, \mathbf V, \mathbf O  \}$ }.
     \STATE{Initialize mask $\mathbf M\in \mathbb R^{B\times B}$, where $\mathbf M_{ts} = 1$, if $t\ge s$, else 0.}
      \STATE{Initialize $\mathbf {KV} =0, \mathbf{dKV}=0\in \mathbb R^{d\times d} $.}
  \FOR{$t=  1,\ldots ,T$}
        \STATE{Load $\mathbf K_t, \mathbf V_t, \mathbf O_t, \mathbf {dO}_t \in \mathbb{R}^{B \times d}$ from HBM to on-chip SRAM.}  
        \STATE{On chip, compute $\mathbf {dQ}_{\mathrm{intra}} =[(\mathbf {dO}_t \mathbf V_t^{\top}) \odot \mathbf M] \mathbf{K}_t$.}
    \STATE{On chip, compute $\mathbf {dQ}_{\mathrm{inter}} =\mathbf{dO}_t \mathbf{KV}^{\top} $.}
     \STATE{On chip, compute $\mathbf{KV} = \mathbf{KV}+  \mathbf K_t^{\top}  \mathbf V_t$.}             
     \STATE{Write $\mathbf {dQ}_t=\mathbf{dQ}_{\mathrm{intra}} + \mathbf{dQ}_{\mathrm{inter}}$ to HBM as the $t$-th block of $\mathbf {dQ}$.}
      \ENDFOR     
    \FOR{$t=  T,\ldots ,1$}
        \STATE{Load $\mathbf Q_t, \mathbf K_t, \mathbf V_t, \mathbf O_t, \mathbf {dO}_t \in \mathbb{R}^{B \times d}$ from HBM to on-chip SRAM.}       
    \STATE{On chip, compute $\mathbf{dK_{\mathrm{intra}}}=[(\mathbf {dO}_t \mathbf V_t^{\top}) \odot \mathbf M ]^{\top} \mathbf{Q}_t$.}
     \STATE{On chip, compute $\mathbf{dK_{\mathrm{inter}}}={ \mathbf V_t}\mathbf{dKV}^{\top}$.}   
  \STATE{On chip, compute $\mathbf{dV_{\mathrm{intra}}}=[(\mathbf Q_t \mathbf K_t^{\top}) \odot \mathbf M ]^{\top} \mathbf{dO}_t$.}
     \STATE{On chip, compute $\mathbf{dV_{\mathrm{inter}}}=\mathbf K_t \mathbf{dKV}$.}    
   \STATE{On chip, compute $\mathbf{dKV} = \mathbf{dKV}+ \mathbf Q_t^{\top}  \mathbf {dO}_t $.}         
     \STATE{Write $ \mathbf {dK}_t=\mathbf {dK}_{\mathrm{intra}} +\mathbf {dK}_{\mathrm{inter}} ,\mathbf {dV}_t=\mathbf {dV}_{\mathrm{intra}} +\mathbf {dV}_{\mathrm{inter}}$ to HBM as the $t$-th block of $ \mathbf {dK}, \mathbf {dV}$.}
      \ENDFOR
      \STATE{Return $\mathbf {dQ, dK, dV}$.}
\end{algorithmic}
\end{algorithm}

For the backward propagation, according to~\citep{katharopoulos2020transformers}, we can rewrite the process as:
\begin{equation*}
\begin{aligned}
&\mathbf{d} \mathbf{q}_t^{\top} =\mathbf{d} \mathbf{o}_t^{\top} \mathbf{k} \mathbf{v}_t^{\top},
\mathbf{d k}_t^{\top} =\mathbf{v}_t^{\top}\mathbf{d} \mathbf{k} \mathbf{v}_t^{\top},
\mathbf{d v}_t^{\top}  =\mathbf{k}_t^{\top}\mathbf{d} \mathbf{k} \mathbf{v}_t, \\
&\mathbf{d k v}_{n+1}=0 \in \mathbb{R}^{d \times d}, 
\mathbf{d} \mathbf{k} \mathbf{v}_{t-1}= \mathbf{d} \mathbf{k} \mathbf{v}_t+\mathbf{q}_{t-1} \mathbf{d} \mathbf{o}_{t-1}^{\top}.
\end{aligned}
\end{equation*}
Therefore, the calculation of the backward propagation is consistent with the forward Eq.~\ref{eq:fwd}, and the Lightning Attention Backward Pass can also be obtained using the tiling technique. A detailed proof can be found in the Appendix~\ref{proof}.

\subsection{Complexity analysis}
\newtheorem{complexity}{Theorem}[section]
\begin{complexity}
\label{th:complexity}
The time complexity of Lightning Attention is $O(nd^2+ nBd)$\footnote{we choose $B \approx d$ in practice, the time complexity is $O(nd^2)$.}.
\end{complexity}
\begin{proof}[Proof of Theorem~\ref{th:complexity}]
\label{proof:complexity}
For the forward pass, according to Algorithm~\ref{algo:Lightning Attention fw pseudo}, each intra part's time complexity is $O(B^2d)$, each inter part's time complexity is $O(B d^2)$, the time complexity of updating $\mathbf {KV}$ is $O(Bd^2)$, so each the time complexity in each loop is $O(B^2 d + Bd^2)$, since we loop for $T=n/B$ times, the total time complexity is $O((B^2d+Bd^2)n/B)=O(nd^2+nBd)$.
Because the computation of the backward pass is similar to that of the forward pass, the time complexity of the backward pass is also $O(nd^2+nBd)$.
\end{proof}

\subsection{Exact IO-aware Implementation}

Lightning Attention employs the above tiling methodology throughout its whole computation process and leverages distinct approaches to optimize the utilization of memory bandwidth between HBM and SRAM within a GPU. Specifically, in each iteration $t$, matrices $\mathbf{Q}_t, \mathbf{K}_t, \mathbf{V}_t$ undergo segmentation into blocks, subsequently transferred to SRAM for computation. The intra- and inter-block operations are segregated, with intra-blocks employing the left product and inter-blocks utilizing the right product. This approach optimally exploits the computational and memory efficiencies associated with the right product, enhancing overall execution speed. The intermediate activation $\mathbf{KV}$ is iteratively saved and accumulated within SRAM. Subsequently, the outputs of intra-blocks and inter-blocks are summed within SRAM, and the results are written back to HBM. The structure of Lightning Attention is illustrated in Fig.~\ref{fig:lightning2}.
The intricate details of the Lightning Attention implementation are explained through Algorithm~\ref{algo:Lightning Attention fw pseudo} for the forward pass and Algorithm~\ref{algo:Lightning Attention bw pseudo} for the backward pass.

\section{TransNormerLLM}
\subsection{The Overall Structure}
\begin{figure}[t]
  \centering
    \includegraphics[width=0.9\columnwidth]{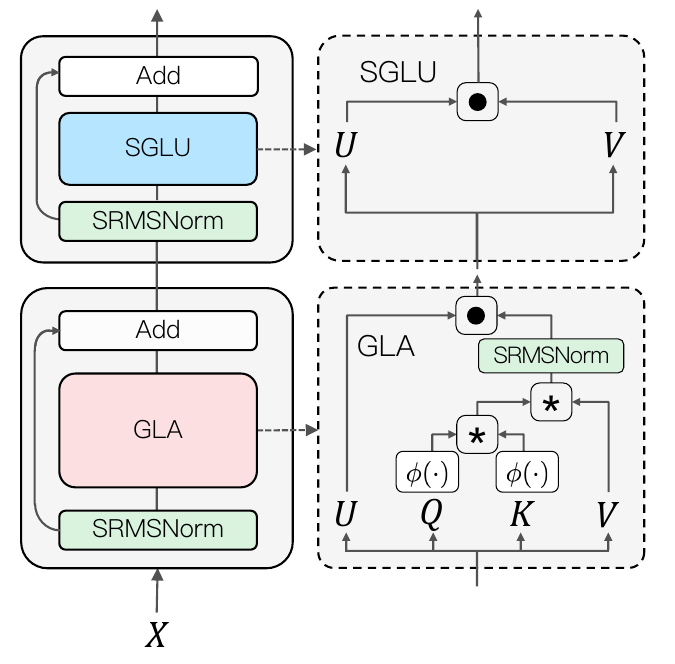}
    \vspace{-3mm}
  \captionof{figure}{Architecture overview of TransNormerLLM (TNL). 
  Each transformer block is composed of a Gated Linear Attention (GLA) for token mixing and a Simple Gated Linear Unit (SGLU) for channel mixing.
  We apply Pre-norm for both modules.}  
  \label{fig:arch}  
  \vspace{4mm}
\end{figure}
Our structure is based on the findings of TransNormer~\cite{qin-etal-2022-devil} but has custom modifications to balance efficiency and performance. We illustrate the overall structure in Fig.~\ref{fig:arch}.
The input $\mathbf X$ is updated through two consecutive steps: 1). It undergoes Gated Linear Attention (GLA) with the application of SimpleRMSNorm (SRMSNorm) normalization. 2). It goes through the Simple Gated Linear Unit (SGLU) with SRMSNorm normalization. We apply the Pre-norm for both modules.

\subsection{Custom Modification}
\normalsize
In this section, we outline the key designs and inspiration behind each custom modification, including positional encoding, gating mechanisms, and tensor normalization.

\textbf{Position Encoding}     
In TransNormer, DiagAttention is used at the lower layers to avoid dilution issues. However, this leads to a lack of global interaction between tokens. 
In TNL, we leverage LRPE~\citep{qin2023linearized} with exponential decay~\citep{alibi,qin2023toeplitz,peng2023rwkv} to address this issue, retaining full attention at the lower layers. The expression of our position encoding is as follows:
\begin{equation}
a_{ts}=\mathbf q_t^{\top} \mathbf k_s \lambda^{t-s}\exp^{i\theta(t-s)}.
\label{eq: pe}
\end{equation}
which we call LRPE-d - Linearized Relative Positional Encoding with exponential decay. Similar to the original LRPE, we set $\theta$ to be learnable. We empirically find that rather than applying LRPE-d to every layer, applying it to the first layer and keeping other layers with exponential decay can speed up training by approximately 15-20\% but only with a subtle effect on the performance.

Note that this position encoding is fully compatible with Linear Attention, as it can be decomposed with respect to $s$ and $t$ separately.
The value of $\lambda$ for the $h$-th head in the $l$-th layer (assuming there are a total of $H$ heads and $L$ layers) is given by:
\begin{equation}
\label{eq:decay}
\textstyle
\lambda =\exp\left(-\frac{8h}{H}\times \left(1-\frac{l}{L}\right) \right).
\end{equation}
Here, $\frac{8h}{H}$ corresponds to the decay rate of the $h$-th head, while $ \left(1-\frac{l}{L}\right)$ corresponds to the decay rate of the $l$-th layer. The term $ \left(1-\frac{l}{L}\right)$ ensures that the Theoretical Receptive Fields (TRF)~\cite{qin2023exploring} at the lower layers is smaller compared to the higher layers, which aligns with TransNormer's motivation. We choose $\lambda$ to be non-learnable since we empirically found that gradients become unstable when $\lambda$ is learnable, leading to NaN values. Note that this positional encoding is still compatible with Lightning Attention, with the specific algorithm detailed in Appendix~\ref{lawd}~\ref{lightning with decay}.

\textbf{Gating Mechanism} Gate can enhance the performance of the model and smooth the training process. In TNL, we adopt the approach from Flash~\citep{hua2022transformer} and use Gated Linear Attention (GLA) in token mixing:
\begin{equation}
\begin{aligned}
\mathbf{O}&=\mathrm{Norm}(\mathbf{Q} \mathbf{K}^{\top}\mathbf{V})\odot \mathbf{U}, \mathbf Q=\phi(\mathbf X \mathbf W_q),\\
\mathbf K&=\phi(\mathbf X \mathbf W_k),\mathbf V=\mathbf X \mathbf W_v,\mathbf U=\mathbf X \mathbf W_u.
\label{eq: gla}
\end{aligned}
\end{equation}
We choose $\phi$ to be Swish~\citep{1710.05941} activation function as we empirically find that it outperforms other activation functions.

\begin{figure*}[t]
\centering
\includegraphics[width=1\textwidth]{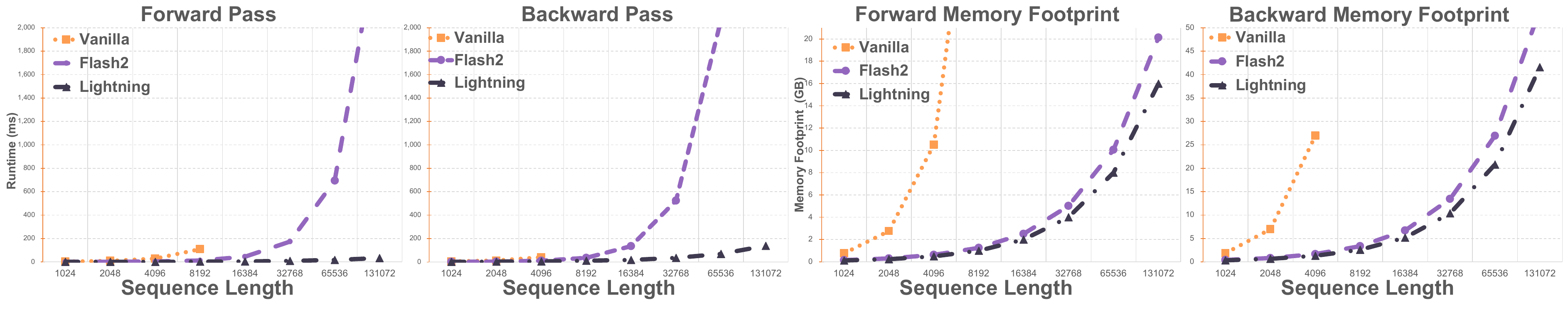}
\caption{\textbf{Comparative Analysis of Speed and Memory Usage}: Vanilla represents norm linear attention in pytorch~\cite{qin-etal-2022-devil}, Flash2 represents FlashAttention-2. Left two sub-figures: Runtime in milliseconds for the forward and backward pass across varying sequence lengths. Right two sub-figures: Memory utilization (in GB) during the forward and backward pass at different sequence lengths.}
\label{fig: timecomp}
\end{figure*}

To further accelerate the model, we propose Simple GLU (SGLU), which removes the activation function from the original GLU structure as the gate itself can introduce non-linearity. Therefore, our channel mixing becomes:
\begin{equation}
\mathbf {O}=[\mathbf V\odot \mathbf U]\mathbf W_o,\\
\mathbf V=\mathbf X \mathbf W_v,\mathbf U=\mathbf X \mathbf W_u.
\label{eq: glu}
\end{equation}
We empirically find that not using an activation function in GLU will not lead to any performance loss.

\textbf{Tensor Normalization}
The origin NormAttention introduced in TransNormer~\citep{qin-etal-2022-devil} is as follows:
\begin{equation}
\mathbf{O}=\mathrm{Norm}(\mathbf{Q} \mathbf{K}^{\top}\mathbf{V})
\end{equation}
In TransNormerLLM, we replace the origin RMSNorm with a new simple normalization function called SimpleRMSNorm, abbreviated as SRMSNorm:
\begin{equation}
\textstyle
\mathrm{SRMSNorm}(\mathbf x)=\frac{\mathbf x}{\|\mathbf x \|_2/\sqrt d}.
\end{equation}
We empirically find that using SRMSNorm does not lead to any performance loss.

\section{Experiments}
We carried out thorough experiments on TNL models and lightning attention. We implemented our models on the Metaseq framework~\cite{zhang2022opt} with Pytorch~\cite{paszke2019pytorch}. The Lightning Attention was executed through Triton~\cite{Tillet2019TritonAI}. All the experiments were conducted on A100 80G GPU clusters. The assessment of our work is divided into three main sections: \rnum{1}) We evaluated the efficiency and accuracy of the Lightning Attention module; \rnum{2}) We further benchmarked our TNL models' performance on standard small-scale corpus and LLM benchmarks and compared their training and inference speeds with STOA models; \rnum{3}) We also provide an ablation study on the design of TNL.

\subsection{Lightning Attention Evaluation}
Since our Lightning Attention is an exact implementation of norm linear attention~\cite{qin-etal-2022-devil}, we compared the speed and memory usage between its original pytorch implementation (named Vanilla) and our Lightning Attention. As a reference, we have also included FlashAttention-2~\cite{dao2023flashattention2} (named Flash2), which is currently the SOTA implementation of softmax attention. As shown in Fig.~\ref{fig: timecomp}, Lightning Attention shows remarkable linear growth of processing time in both forward and backward passes, whereas Vanilla and Flash2 exhibit quadratic growth. In terms of memory footprint, Vanilla tends to rapidly exhaust memory resources. Lightning Attention shows a similar trend to Flash2 but requires less memory.

\begin{table}[t]
    \centering
    \small 
    \vspace{-2mm}
    \caption{\textbf{Results on Wikitext-103} (TNN\cite{qin2023toeplitz}'s setting). $\downarrow$ means \textit{lower is better}.}
     \setlength{\tabcolsep}{0.28cm}
     \label{lm}
    \begin{tabular}{ll|lll}
    \toprule
        & Model & \makecell[c]{PPL \\(val)$\downarrow$} & \makecell[c]{PPL \\(test)$\downarrow$} &  \makecell[c]{Params\\(M)} \\ \hline
        \multirow{7}{*}{\textit{Attn-based}}
        & Transformer & 24.40 & {24.78} & 44.65 \\ 
        & FLASH & 25.92 & 26.70 & 42.17 \\ 
        & 1+elu & 27.44 & 28.05 & 44.65 \\ 
        & Performer & 62.50 & 63.16 & 44.65 \\
        & cosFormer & 26.53 & 27.06 & 44.65 \\ 
        & {TN1} & 24.43 & 25.00 &44.64\\ 
        & {TN2} & 24.50 & 25.05 &44.64 \\           \hline
        \multirow{3}{*}{\textit{MLP-based}}
        & Syn(D) & 31.31 & 32.43 & 46.75 \\ 
        & Syn(R) & 33.68 & 34.78 & 44.65 \\ 
        & gMLP   & 28.08 & 29.13 & 47.83 \\ \hline
        \multirow{6}{*}{\textit{RNN-based}}
        & S4    & 38.34 & 39.66 & 45.69 \\ 
        & DSS   & 39.39 & 41.07 & 45.73 \\ 
        & GSS   & 29.61 & 30.74 & 43.84 \\ 
        & RWKV  & 24.31 & 25.07 & 46.23\\
        & LRU  & 29.86 & 31.12 & 46.24\\
        & {HGRN} & {24.14} & 24.82 &46.25 \\ \hline
       \textit{FFT-based} & TNN  & {23.98} & {24.67} & 48.68 \\ \hline   
        \textit{Ours} & TNL & {23.46} & 24.03 &45.45 \\
    \bottomrule
    \end{tabular}
    \label{tab:std_benchmark}
    \vspace{3mm}
\end{table}

\definecolor{Gray}{gray}{0.95}
\begin{table*}[t]
    \centering
    \small
    \vspace{-2mm}
    \caption{\textbf{Performance Comparison on Commonsense Reasoning and Aggregated Benchmarks.} For a fair comparison, we report competing methods' results reproduced by us using their released models. 
    Official results are denoted in \textit{italics}. PS: parameter size (billion). T: tokens (billion). HS: HellaSwag. WG: WinoGrande. }
    \vspace{3mm}
    \setlength{\tabcolsep}{1.8mm}
        \begin{tabular}{p{2.3cm}|cc|ccccccc|cc}
        \toprule
            Model & PS & T & BoolQ & PIQA & HS & WG & ARC-e & ARC-c & OBQA & MMLU & C-Eval \\ \hline
             & B & B &acc &acc &acc\_norm &acc &acc &acc\_norm &acc\_norm  &acc-5shot &acc-5shot \\ \midrule

            OPT         & 0.35 & 0.30 & 57.74 & 64.58 & 36.69 & 52.49 & 44.02 & 23.89 & 28.20 & 26.02 & 25.71 \\ 
            Pythia      & 0.40 & 0.30 & 60.40 & 67.08 & 40.52 & 53.59 & 51.81 & 24.15 & 29.40 & 25.99 & 24.81 \\ 
            RWKV        & 0.43 & - & - & \textit{67.52} & \textit{40.90} & \textit{51.14} & \textit{52.86} & \textit{25.17} & \textit{32.40} & 24.85 & - \\
    \rowcolor{Gray} TNL & 0.39 & 1.0 &62.14 &66.70  &46.27	&54.46	 &55.43 &27.99 &32.40 &25.90 & 25.24 \\  \hline

            OPT     & 1.3 & 0.3 & 57.77 & 71.71 & 53.70 & 59.35 & 57.24 & 29.69 & 33.20 & 24.96 & 25.32 \\ 
            Pythia  & 1.4 & 0.3 & 60.73 & 70.67 & 47.18 & 53.51 & 56.99 & 26.88 & 31.40 & 26.55  & 24.25 \\
            RWKV    & 1.5 &- & - & \textit{72.36} & \textit{52.48} & \textit{54.62} & \textit{60.48} & \textit{29.44} & \textit{34.00}  & 25.77 & - \\ 
            Falcon  & 1.0 & 0.35 & 61.38 & 75.14 & 61.50 & 60.30 & 63.38 & 32.17 & 35.60 & 25.28 & 25.66 \\ 
    \rowcolor{Gray} TNL & 1.0 & 1.2 & 63.27 & 72.09 & 56.49 & 60.38 & 63.68 & 35.24 & 36.60  & 27.10 &26.01\\ \hline

        OPT & 6.7 & 0.3 & 66.18 & 76.22 & 67.21 & 65.19 & 65.66 & 34.64 & 37.20 & 24.57& 25.32 \\ 
        Pythia & 6.9 & 0.3 & 63.46 & 75.14 & 63.92 & 60.77 & 67.34 & 35.41 & 37.00  & 24.64 & 26.40 \\ 
        RWKV & 7.4 & - & - & \textit{76.06} & \textit{65.51} & \textit{61.01} & \textit{67.80} & \textit{37.46} & \textit{40.20} & 24.96 & - \\ 
        Falcon & 7.2 & 1.5 & 73.73 & 79.38 & 76.3 & 67.17 & 74.62 & 43.60 & 43.80  & 27.79 & 22.92 \\
        Baichuan2 & 7.0 & 2.6 & 72.72 & 76.50 & 72.17 & 68.35 & 75.17 & 42.32 & 39.60 & \textit{54.16}  & \textit{54.00} \\
        ChatGLM2 & 7.1 & 1.4 & 77.65 & 69.37 & 50.51 & 57.62 & 59.13 & 34.30 & 37.00 & \textit{45.46}  & \textit{52.55} \\
       
        OpenLLaMAv2 & 6.7 & 1.0 & 72.20 & 78.84 & 74.51 & 65.67 & 72.39 & 41.30 & 41.00  & 41.29  & 30.01 \\ 
        LLaMA1 & 6.7 & 1.0 & \textit{76.50} & \textit{79.80} & \textit{76.10} & \textit{70.10} & \textit{72.80} & \textit{47.60} & \textit{57.20}  & \textit{35.10} & 25.72 \\ 
        LLaMA2 & 6.7 & 2.0  & \textit{77.68} & \textit{78.07} & \textit{76.02} & \textit{68.98} & \textit{76.30} & \textit{46.33} & \textit{44.20}   & \textit{45.30} & 33.20 \\
        \rowcolor{Gray}
       TNL & 6.8 & 1.4 &75.87&	80.09 &75.21 & 66.06 &75.42 & 44.40 & 63.40  & 43.10 & 43.18 \\ \hline
       OPT          & 13  & 0.3 & 65.93 & 75.84 & 69.83 & 65.19 & 67.00 & 35.75 & 38.80  & 24.68  & 22.23 \\
       Pythia       & 12  & 0.3 & 65.72 & 76.17 & 68.85 & 66.22 & 70.62 & 38.23 & 41.00  & 25.51  & 22.99 \\
       RWKV         & 14  & -   & 70.12 & 78.51 & 71.49 & 64.48 & 72.35 & 40.87 & 41.00  & 26.49  & 26.49 \\
       Baichuan2    & 13  & 2.6 & 79.20 & 77.31 & 75.27 & 70.01 & 77.36 & 47.01 & 43.80  & 57.02  & 59.63 \\
       OpenLLaMAv2    & 13  & 1.0 & 72.29 & 77.58 & 72.07 & 70.09 & 75.42 & 43.86 & 43.00  & 43.43  & 25.95 \\
       LLaMA1       & 13  & 1.0 & 77.95 & 79.16 & 79.06 & 72.61 & 77.40 & 47.70 & 44.80  & 47.62  & 32.13  \\
       LLaMA2       & 13  & 2.0 & 80.61 & 79.11 & 79.35 & 72.38 & 79.34 & 48.98 & 35.20  & 55.70  & 38.34  \\

\rowcolor{Gray} TNL & 15  & 2.0  &76.64	&81.56	&82.18	&75.61	&77.61	&50.51	&46.40 & 60.06 & 53.01  \\  

        \bottomrule              
        \end{tabular}
    \label{tab:benchmark}
\end{table*}


\subsection{TNL Evaluation}
\textbf{Performance Evaluation}
In Table~\ref{tab:std_benchmark}, we present an evaluation across various 40M models on a standard dataset. This includes models based on attention/linear attention mechanisms~\cite{vaswani2017attention,dao2022flashattention,katharopoulos2020transformers,zhen2022cosformer,qin-etal-2022-devil}, MLPs (Multi-Layer Perceptrons)~\cite{tay2021synthesizer,liu2021pay}, RNNs (Recurrent Neural Networks)~\cite{gu2022efficiently,gupta2022diagonal,mehta2022long,peng2023rwkv,lru}, FFTs (Fast Fourier Transforms)~\cite{qin2023toeplitz}, and our model. TNL records the lowest perplexity on test set after trained on the Wikitext-103 dataset.

We also scaled up our model to 1B and 3B parameters and compared its training loss with top-tier LLM structures such as LLaMA-FA2~\cite{touvron2023llama, dao2023flashattention2}, HGRN~\cite{qin2023hierarchically}, and TNN~\cite{qin2023toeplitz}. For a fair comparison, we retrain all models on the same 30B corpus and plot the training losses in Fig.~\ref{fig: tgs}. TNL achieved the lowest training losses in both 1B and 3B parameters.

\textbf{Efficiency Evaluation}
In Fig.~\ref{fig: tgs}, we present a comparative analysis of training speeds under the same corpora and hardware setups. This comparison encompasses four variants: TNL, LLaMA-FA2~\cite{touvron2023llama, dao2023flashattention2}, HGRN~\cite{qin2023hierarchically} , and TNN~\cite{qin2023toeplitz}. Our findings show that during both the forward and backward passes, the TGS (tokens per GPU per second) for TNL remains consistently high, while the other three models exhibit a rapid decline when sequence length is scaled from 1K to 128K. This pattern suggests that Lightning Attention offers a significant advancement in managing extremely long sequence lengths in LLM.

\begin{figure}[t]
    \centering
\includegraphics[width=1\columnwidth]{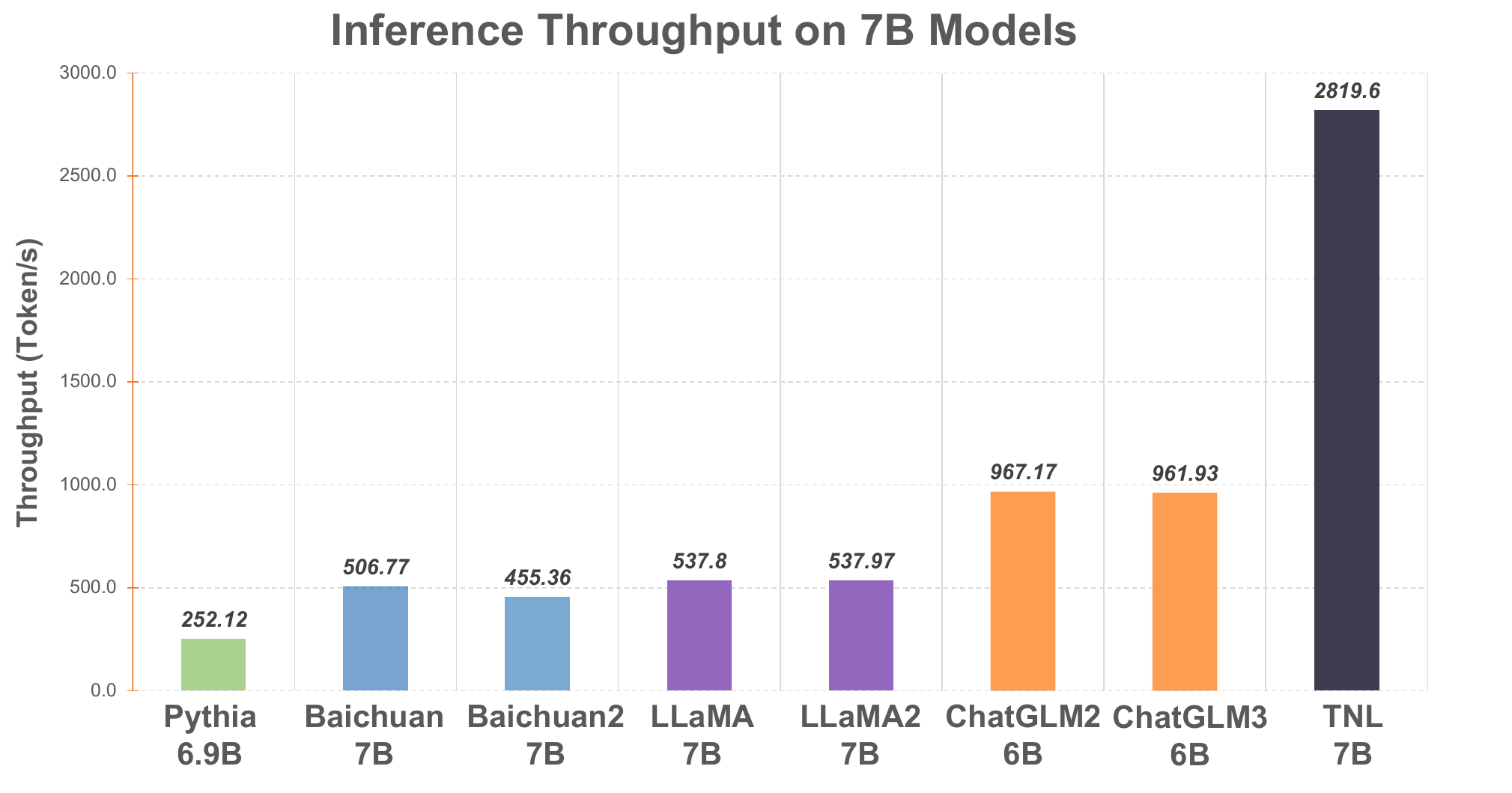}
    \vspace{-6mm}
    \caption{\textbf{Inference Throughput Comparison.} We measure the inference throughput of various 7B LLM models on a A100 80G GPU. Batch sizes for models are chosen to optimize GPU utilization without exceeding memory limits. Each model is tested with a 512-token input prompt and can generate up to 1024 new tokens. Reported throughput is averaged from 20 attempts.}
    \label{fig: throughput-7b}
    \vspace{3mm}
\end{figure}

\textbf{Inference Evaluation}
We conduct an inference throughput comparison on various 7B large language models using their standard codebase from Huggingface, as detailed in Fig.~\ref{fig: throughput-7b}. TNL with Lightning Attention demonstrates a significant advantage, achieving a throughput rate that up to 11$\times$ higher than transformer structure models. 

\textbf{Benchmark Results}
In order to validate the effectiveness of TNL, we pretraining 385M, 1B, 7B, and 15B models on self-collected datasets, the details of the data are in the Appendix~\ref{app:corpus}, and tested on Commonsense Reasoning Task, MMLU\citep{hendrycks2021measuring}, C-Eval\citep{huang2023ceval}, and SCROLLS~\cite{shaham2022scrolls}. For comparison, we selected several open-source models as competitors, including Transformer-based models such as OPT~\citep{zhang2022opt}, Pythia~\citep{biderman2023pythia}, BLOOM~\citep{workshop2023bloom}, GPT-Neo~\citep{gpt-neo}, Falcon~\citep{falcon40b}, LLaMA~\citep{touvron2023llama,2307.09288}, OpenLLAMA ~\citep{openlm2023openllama}, Baichuan~\citep{baichuan2023baichuan2}, ChatGLM~\citep{zeng2022glm,du2022glm}, and non-Transformer model RWKV~\citep{2305.13048}. It can be observed in Table~\ref{tab:benchmark} and Table~\ref{tab:benchmark_scrolls} that, compared to these models, TNL remains highly competitive.
\begin{compactitem}
    \item We report BoolQ \citep{clark2019boolq}, PIQA \citep{bisk2019piqa}, SIQA \citep{sap2019socialiqa}, HellaSwag \citep{zellers2019hellaswag}, WinoGrande \citep{sakaguchi2019winogrande}, ARC easy and challenge \citep{clark2018think} and OpenBookQA \citep{mihaylov2018suit}. We report 0-shot results for all benchmarks using LM-Eval-Harness \citep{leo2021evalharness}. All of our models achieve competitive performance compared to existing state-of-the-art LLMs, showcasing a remarkable ability to comprehend and apply commonsense reasoning.
    \item We report the overall results for MMLU \citep{hendrycks2021measuring}, C-Eval \citep{huang2023ceval}. Official scripts were used for evaluating MMLU and C-Eval, with all evaluation results being conducted with a 5-shot setup. In comparison to top-tier open-source models available in the industry, our models have demonstrated matched performance in both English and Chinese benchmarks.
    \item On SCROLLS~\cite{shaham2022scrolls} benchmark, we assess the large language models trained on a 1 billion parameter and pre-trained using a sequence length of 2048. We present zero-shot performance results for all benchmarks using the LM-Eval-Harness \citep{leo2021evalharness}. For generation tasks within SCROLLS, we employ a greedy search with hyper-parameters top\_k set to 5 and top\_p set to 1. Our models consistently match or surpass the performance of existing state-of-the-art LLMs in these tasks.
\end{compactitem}

\subsection{TNL Ablation}
We conducted an extensive ablation analysis on various components of TNL, including positional encoding, gating mechanisms, GLA activation functions, GLU activation functions, and normalization functions.

\definecolor{Gray}{gray}{0.95}
\begin{table*}[t]
    \centering
    \small
    \vspace{-2mm}
    \caption{\textbf{Performance Comparison on SCROLLS~\cite{shaham2022scrolls}:} A review of models up to 1 billion parameters on 2048 pre-training sequence length. PS: parameter size (billion). T: tokens (billion).}
    \vspace{2mm}
    \setlength{\tabcolsep}{2mm}
        \begin{tabular}{p{1.9cm}|cc|ccccccc|c}
        \toprule
            Model & PS & T & GovRep & SumScr & QMSum &Qspr & Nrtv & QALT & CNLI & Avg \\ \hline
             & B & B  &ROUGE-1/2/L &ROUGE-1/2/L &ROUGE-1/2/L &F1 &F1 &EM &EM & \\ \midrule

        OPT         &0.35 &0.30  &2.52/0.53/2.24 &7.72/0.68/6.52 &8.05/1.79/6.6 &13.13 &10.13 &29.05 &9.16 &7.55 \\
        Pythia      &0.40 &0.30  &4.96/1.19/4.06 &2.03/0.2/1.79 &7.51/1.43/6.08 &15.27 &8.24 &28.57 &15.24 &7.43 \\
        RWKV        &0.43 &-  &1.63/0.4/1.49 &0.94/0.11/0.76 &10.19/2.26/8.06 &13.16 &9.76 &26.32 &16.49 &7.04 \\
        \rowcolor{Gray} 
        TNL         &0.39 &1.0  &3.67/1.16/3.14 &8.27/0.82/6.91 &13.62/3.29/10.95 &14.29 &11.69 &28.14 &17.36 &9.48 \\ \hline

        OPT     &1.3 &0.3  &5.7/2.09/4.41 &10.17/0.82/8.29 &12.36/3.15/9.85 &18.37 &13.42 &29.15 &12.44 &10.02 \\
        Pythia  &1.4 &0.3  &4.03/1.25/3.33 &8.34/0.87/6.97 &13.17/3.4/10.92 &16.09 &11.91 &28.72 &9.06 &9.08 \\ 
        Falcon  &1.0 &0.35 &2.74/0.67/2.37 &10.95/1.28/8.66 &13.29/3.09/10.58 &16.17 &12.91 &29.19 &14.75 &9.74 \\
        \rowcolor{Gray}
        TNL &1.0 &1.2  &6.81/2.30/5.25 &12.28/1.23/9.27 &14.60/3.51/11.62 &15.02 &14.66 &28.72 &37.32 &12.51 \\

        \bottomrule              
        \end{tabular}
    \label{tab:benchmark_scrolls}
\end{table*}

\begin{table}[H]
    \centering
  \small
  \vspace{-3mm}
    \caption{\textbf{Exploration of Positional Encoding.} LRPE-d leads to the most optimal outcome.}
    \vspace{1mm}
    \label{tab:pe}
    \centering
    \setlength{\tabcolsep}{3.2mm}
        \begin{tabular}{lllll}
        \toprule
        PE Methods & Params & Updates & Loss &PPL\\ \hline
         Mix & 385M  & 100K& 2.248 &4.770\\ 
            APE & 386M & 100K & 2.387 & 5.253\\
            Exp-Decay & 385M & 100K & 2.267 & 4.834 \\ 
            LRPE & 385M & 100K & 2.287 & 4.899 \\             
            LRPE-d & 385M  & 100K & 2.236 &4.728  \\ 
   \bottomrule
    \end{tabular}
\end{table}

\textbf{Positional Encoding:} in our experiment comparing various PE strategies—Mix, Absolute Positional Encoding (APE), LRPE, Exponential Decay, and LRPE-d—our approach and LRPE-d demonstrated superior performance. We chose the Mix method for its ability to enhance training speed by up to 20\%, despite being slightly less effective than LRPE-d.

\begin{table}[H]
    \centering
    \vspace{2mm}
  \small    \caption{\textbf{Ablations on decay temperature.} The results of decay temperature proved to be superior.}
  \vspace{1mm}
    \label{tab:temperature}
    \centering
    \setlength{\tabcolsep}{2.7mm}
        \begin{tabular}{lllll}
        \toprule
            Temperature & Params & Updates & Loss &PPL \\ \hline
            w/ temperature &  385M & 100K& 2.248 &4.770\\ 
            w/o temperature & 385M & 100K & 2.258 &4.804 \\
        \bottomrule
    \end{tabular}
\end{table}
We also perform ablations on the decay temperature $\left(1-\frac{l}{L}\right)$ in Eq.~\ref{eq:decay}. The perplexity of the TNL is reduced by adding the decay temperature, as shown in Table~\ref{tab:temperature}.

\begin{table}[H]
    \centering
  \small
  \vspace{-4mm}
    \caption{\textbf{Ablations on gating mechanism.} The performance with the gate proved to be superior.}
    \vspace{1mm}
    \label{tab:gate}
    \centering
    \setlength{\tabcolsep}{3.8mm}
        \begin{tabular}{lllll}
        \toprule
            Gate & Params & Updates & Loss &PPL\\ \hline
            w/ gate & 385M & 100K & 2.248 &4.770\\ 
            w/o gate & 379M & 100K & 2.263 &4.820 \\
        \bottomrule
    \end{tabular}
\end{table}

\textbf{Gating Mechanism:} we further investigate the impact of integrating a gating mechanism. According to the data presented in Table~\ref{tab:gate}, enabling the gate decreased the loss value from 2.263 to 2.248.

\begin{table}[h]
    \centering
  \small
  \vspace{-2mm}
    \caption{\textbf{Exploration of Normalization Function.} The deviation in results among the bellowing normalization functions is minimal.}
    \vspace{1mm}
    \label{tab:norm}
    \centering
    \setlength{\tabcolsep}{3.3mm}
     \begin{tabular}{lllll}
       \toprule
        Norm Type & Params & Updates & Loss &PPL \\ \hline
        SRMSNorm & 385M  & 100K & 2.248 &4.770 \\  
        RMSNorm & 385M & 100K & 2.247  & 4.766\\ 
        LayerNorm & 385M & 100K & 2.247 &4.765 \\ 
        \bottomrule
    \end{tabular}
    \vspace{4mm}
\end{table}

\vspace{2mm}
\textbf{Normalization Functions:} our study involved testing various normalization techniques—SRMSNorm, RMSNorm, and LayerNorm—on TNL, finding little difference in their effectiveness. However, we enhanced SRMSNorm using Triton, resulting in notable improvements in processing speed for larger dimensions.

\vspace{3mm}
\textbf{GLA Activation Functions:} in our study on the GLA (Gated Linear Attention) mechanism, we evaluated activation functions, finding Swish and 1+elu to perform similarly, as detailed in Table ~\ref{tab:gla_act}. However, due to NaN issues with 1+elu in our 7B model, we opted for Swish.

\begin{table}[h]
    \centering
  \small
  \vspace{-4mm}
    \caption{\textbf{Ablations on GLA activation functions.} The results obtained from different activation functions were virtually identical.}
    \vspace{1mm}
    \label{tab:gla_act}
    \centering
    \setlength{\tabcolsep}{3.6mm}
         \begin{tabular}{lllll}
       \toprule
        GLA Act & Params & Updates & Loss &PPL \\ \hline
        Swish & 385M & 100K & 2.248 &4.770\\
        No Act & 385M & 100K & 2.283 &4.882 \\ 
        1+elu	& 385M	& 100K	& 2.252	& 4.767 \\
        \bottomrule
    \end{tabular}
\end{table}

\textbf{GLU Activation Functions:} our experiment additionally involved removing the activation function from the Gated Linear Units (GLU), showing minimal effect on outcomes as per Table~\ref{tab:glu_act}. Therefore, we opted for the Simple Gated Linear Units (SGLU) configuration in our model.

\begin{table}[h]
    \centering
  \small
    \vspace{-4mm}
    \caption{\textbf{Ablations on GLU activation functions.} The exclusion of the activation function had no negative impact on the results.}
    \vspace{1mm}
    \label{tab:glu_act}
    \centering
    \setlength{\tabcolsep}{3.6mm}
         \begin{tabular}{lllll}
        \toprule
        GLU Act & Params & Updates & Loss &PPL \\ \hline
        No Act & 385M   & 100K & 2.248 &4.770 \\  
        Swish & 385M & 100K & 2.254 & 4.788\\ 
    \bottomrule
    \end{tabular}
\end{table}

\vspace{6mm}
\section{Conclusion}
\label{sec: conclusion}
We introduced Lightning Attention, the first linear attention implementation that unleashed the full power of linear attention. As a result, our Lightning Attention can handle various sequence lengths with a constant speed under a constant memory footprint. 
The main concept is to divide the calculation of attention into intro-blocks and inter-blocks, while applying distinct computation techniques to perform the calculation.
A new architecture, TNL, that is tailored for Lightning Attention is presented. TNL outperforms existing efficient language models in terms of both efficiency and accuracy and achieves competitive performance compared to state-of-the-art large language models using conventional transformer architectures.

\vspace{3mm}
\section*{Acknowledgement}
This work is partially supported by the National Key R\&D Program of China (NO.2022ZD0160100). We thank Songlin Yang for the helpful discussions.

\vspace{3mm}
\section*{Impact Statement}
\label{sec: statement}
The introduction of Lightning Attention and its accompanying architecture TNL, heralds significant shifts in machine learning, particularly in language model efficiency and accessibility. By addressing the limitations of linear attention in varying sequence lengths without increasing memory consumption, this advancement democratizes access to state-of-the-art language models, potentially reducing the computational and environmental footprint of large-scale AI systems. Ethically, it underscores a move towards more sustainable AI practices, yet raises questions about the proliferation of powerful language models and their societal impacts, including concerns over privacy, misinformation, and the digital divide.

\bibliography{lightning2}
\bibliographystyle{icml2024}

\appendix
\newpage
\begin{center}
\textbf{\large Appendix}
\end{center}

\section{Linear Attention with decay}
\label{lawd}
TransNormerLLM uses LRPE-d positional encoding, which has the following format:
\begin{equation}
\small
a_{ts}=\mathbf q_t^{\top} \mathbf k_s \lambda^{t-s}\exp^{i\theta(t-s)}.
\end{equation}
According to~\citep{qin2023linearized}, Lrpe can be decomposed into $\mathbf q$ and $\mathbf k$, so we consider the following simplified form:
\begin{equation}
\label{eq:left}
\small
\begin{aligned}
a_{ts}&=\mathbf q_t^{\top} \mathbf k_s \lambda^{t-s},\\
\mathbf o_t^{\top} &= \sum_{s=1}^t a_{ts}\mathbf v_t^{\top}\\
&=\sum_{s=1}^t \mathbf q_t^{\top} \mathbf k_s \lambda^{t-s}\mathbf v_s^{\top} \\
&=\mathbf q_t^{\top} \sum_{s=1}^t  \mathbf k_s \lambda^{t-s}\mathbf v_s^{\top} \\
&\triangleq \mathbf q_t^{\top} \overline{\mathbf {kv}}_t.
\end{aligned}
\end{equation}
We call this Linear Attention with decay and prove it's  equivalent to the recurrence form:
\begin{equation}
\label{eq:right}
\mathbf {kv}_0=0, \mathbf {kv}_t=\lambda \mathbf {kv}_{t-1} + \mathbf k_t\mathbf v_t^\top, \mathbf o_t^{\top} = \mathbf q_t^{\top} \mathbf {kv}_t.
\end{equation}

We will use induction to prove $\overline{\mathbf{kv}}_t = \mathbf{kv}_t$.

\small
    \textbf{Base Case ($n=1$):}
    \begin{equation}
    \begin{aligned}
    \overline{\mathbf{kv}}_1 &= \mathbf {k}_1\mathbf {v}_1^{\top}= \mathbf{kv}_1 .
    \end{aligned}
    \end{equation}
Assume the statement holds for $n=m-1$, i.e., ${\mathbf {\overline{kv}}}_{m-1}=\mathbf {kv}_{m-1}$. Then, when $n=m$:
    \begin{equation}
    \begin{aligned}
    {\mathbf {\overline{kv}}}_{m} &= \sum_{s=1}^m  \mathbf k_s \lambda^{m-s}\mathbf v_s^{\top} \\
&=\lambda \sum_{s=1}^{m-1}  \mathbf k_s \lambda^{m-1-s}\mathbf v_s^{\top} + \mathbf k_m \mathbf v_m^{\top} \\
    &=\lambda {\mathbf {\overline{kv}}}_{m-1}+\mathbf k_m \mathbf v_m^{\top} \\
&= \lambda {\mathbf {{kv}}}_{m-1}+\mathbf k_m \mathbf v_m^{\top}  \\
    &={\mathbf {{kv}}}_{m},
    \end{aligned}
    \end{equation}
\normalsize
the statement holds. Therefore, by induction, the statement holds for all $n\geq 1$.


\section{Lightning Attention with decay}
\label{lightning with decay}
We extended Lightning Attention to accommodate Linear Attention with decay. The complete algorithm can be found in Algorithm~\ref{algo:Lightning Attention with decay fw pseudo},~\ref{algo:Lightning Attention with decay bw pseudo}, and the proof of correctness is provided in~\ref{proof}.

\begin{algorithm}[t]
\small
    \caption{Lightning Attention(with decay) Forward Pass}
    \label{algo:Lightning Attention with decay fw pseudo}
    \begin{algorithmic}
    \STATE{\textbf{Input:} $\mathbf Q,\mathbf K,\mathbf V \in \mathbb{R}^{n \times d}$, decay rate $\lambda \in \mathbb R^+$, block sizes $B$.}
    \STATE{Divide $\mathbf {X}$ into $T = \frac{n}{B}$ blocks $\mathbf X_1, \mathbf X_2, ...\mathbf X_{T}$ of size $B \times d$ each, where $\mathbf X\in \{\mathbf Q, \mathbf K, \mathbf V,\mathbf O \}$. }
     \STATE{Initialize mask $\mathbf M\in \mathbb R^{B\times B}$, where $\mathbf M_{ts} = \lambda^{t-s}$, if $t\ge s$, else 0.}
     \STATE{Initialize $\Lambda =\mathrm{diag}\{\lambda, \lambda^2,\ldots, \lambda^{B}\}\in \mathbb R^{B\times B} $.} 
     \STATE{Initialize $\mathbf {KV} =0\in \mathbb R^{d\times d} $.}
    \FOR{$t=  1,\ldots ,T$}
        \STATE{Load $\mathbf Q_t,\mathbf  K_t, \mathbf V_t \in \mathbb{R}^{B \times d}$ from HBM to on-chip SRAM.}
        \STATE{On chip, compute $\mathbf O_{\mathrm{intra}}= [(\mathbf Q_t \mathbf K_t^{\top }) \odot \mathbf M]\mathbf V_t$.}
        \STATE{On chip, compute $\mathbf{O}_{\mathrm{inter}} =\Lambda \mathbf Q_t (\mathbf {KV}) $.}
        \STATE{On chip, compute $\mathbf{KV} =\lambda^B \mathbf{KV}+ (\lambda^{B}\Lambda^{-1} \mathbf K_t)^{\top}  \mathbf V_t$.}
      \STATE{Write $\mathbf O_t=\mathbf O_{\mathrm{intra}}+ \mathbf{O}_{\mathrm{inter}}$ to HBM as the $t$-th block of $\mathbf O$.}
      \ENDFOR
      \STATE{return $\mathbf O$.}
\end{algorithmic}
\end{algorithm}

\begin{algorithm}[t]
\small
    \caption{Lightning Attention(with decay) Backward Pass}
    \label{algo:Lightning Attention with decay bw pseudo}
    \begin{algorithmic}
    \STATE{\textbf{Input:} $\mathbf Q,\mathbf K,\mathbf V,\mathbf{dO} \in \mathbb{R}^{n \times d}$, decay rate $\lambda \in \mathbb R^+$, block sizes $B$.}
    \STATE{Divide $\mathbf {X}$ into $T = \frac{n}{B}$ blocks $\mathbf X_1, \mathbf X_2, ...\mathbf X_{T}$ of size $B \times d$ each, where $\mathbf X\in \{\mathbf Q, \mathbf K, \mathbf V \}$. }
     \STATE{Divide $\mathbf {dX}$ into $T = \frac{n}{B}$ blocks $\mathbf {dX}_1, \mathbf {dX}_2, ...\mathbf {dX}_{T}$ of size $B \times d$ each, where $\mathbf X\in \{\mathbf Q, \mathbf K, \mathbf V, \mathbf O  \}$ }.
     \STATE{Initialize mask $\mathbf M\in \mathbb R^{B\times B}$, where $\mathbf M_{ts} = \lambda^{t-s}$, if $t\ge s$, else 0.}
     \STATE{Initialize $\Lambda =\mathrm{diag}\{\lambda, \lambda^2,\ldots, \lambda^{B}\}\in \mathbb R^{B\times B} $} .
      \STATE{Initialize $\mathbf {KV} =0, \mathbf{dKV}=0\in \mathbb R^{d\times d} $.}

  \FOR{$t=  1,\ldots ,T$}
        \STATE{Load $\mathbf K_t, \mathbf V_t, \mathbf O_t, \mathbf {dO}_t \in \mathbb{R}^{B \times d}$ from HBM to on-chip SRAM.}
    
        \STATE{On chip, compute $\mathbf {dQ}_{\mathrm{intra}} =[(\mathbf {dO}_t \mathbf V_t^{\top}) \odot \mathbf M] \mathbf{K}_t$.}

    \STATE{On chip, compute $\mathbf {dQ}_{\mathrm{inter}} =\Lambda \mathbf{dO}_t (\mathbf{KV})^{\top} $.}

     \STATE{On chip, compute $\mathbf{KV} =\lambda^B \mathbf{KV}+ (\lambda^{B}\Lambda^{-1} \mathbf K_t)^{\top}  \mathbf V_t$.}

     \STATE{Write $\mathbf {dQ}_t=\mathbf{dQ}_{\mathrm{intra}} + \mathbf{dQ}_{\mathrm{inter}}$ to HBM as the $t$-th block of $\mathbf {dQ}$.}
      \ENDFOR

    \FOR{$t=  T,\ldots ,1$}
        \STATE{Load $\mathbf Q_t, \mathbf K_t, \mathbf V_t, \mathbf O_t, \mathbf {dO}_t \in \mathbb{R}^{B \times d}$ from HBM to on-chip SRAM.}
        
    \STATE{On chip, compute $\mathbf{dK_{\mathrm{intra}}}=[(\mathbf {dO}_t \mathbf V_t^{\top}) \odot \mathbf M ]^{\top} \mathbf{Q}_t$.}
     \STATE{On chip, compute $\mathbf{dK_{\mathrm{inter}}}={(\lambda^{B} \Lambda^{-1} \mathbf V_t)}(\mathbf{dKV})^{\top}$.}
     
  \STATE{On chip, compute $\mathbf{dV_{\mathrm{intra}}}=[(\mathbf Q_t \mathbf K_t^{\top}) \odot \mathbf M ]^{\top} \mathbf{dO}_t$.}
     \STATE{On chip, compute $\mathbf{dV_{\mathrm{inter}}}=(\lambda^{B} \Lambda^{-1}\mathbf K_t) \mathbf{dKV}$.}
     
   \STATE{On chip, compute $\mathbf{dKV} =\lambda^B \mathbf{dKV}+ (\Lambda \mathbf Q_t)^{\top}  \mathbf {dO}_t $.}

     \STATE{Write $ \mathbf {dK}_t=\mathbf K_{\mathrm{intra}} +\mathbf K_{\mathrm{inter}} ,\mathbf {dV}_t=\mathbf V_{\mathrm{intra}} +\mathbf V_{\mathrm{inter}}$ to HBM as the $t$-th block of $ \mathbf {dK}, \mathbf {dV}$.}
      \ENDFOR
      \STATE{return $\mathbf {dQ, dK, dV}$.}
\end{algorithmic}
\end{algorithm}

\section{Proofs}
\label{proof}
Here we discuss linear attention with decay directly, because vanilla linear attention is the case of $\lambda =1$.
\subsubsection{Forward Pass}
During forward pass of Linear attention with decay, the $t$-th output can be formulated as
\begin{equation}
\mathbf{o}_t^{\top} = \mathbf{q}_t^{\top} \sum_{s \le t} \lambda^{t-s} \mathbf{k}_s \mathbf{v}_s^{\top}.
\end{equation}

In a recursive form, the above equation can be rewritten as
\begin{equation}
\begin{aligned}
\mathbf{kv}_0 &= 0 \in \mathbb{R}^{d \times d}, \\
\mathbf{kv}_t &= \lambda \mathbf{kv}_{t-1} + \mathbf{k}_t \mathbf{v}_t^{\top}, \\
\mathbf{o}_t^{\top} &= \mathbf{q}_t^{\top} (\mathbf{kv}_t),
\end{aligned}
\end{equation}
where
\begin{equation}
\mathbf{kv}_t = \sum_{s \le t} \lambda^{t-s} \mathbf{k}_s \mathbf{v}_s^{\top}.
\end{equation}
To perform tiling, let us write the equations in block form. Given the total sequence length $n$ and block size $B$, $\mathbf{X}$ is divided into $T = \frac{n}{B}$ blocks $\{\mathbf{X}_1, \mathbf{X}_2, \ldots, \mathbf{X}_T\}$ of size $B \times d$ each, where $\mathbf{X} \in \{\mathbf{Q}, \mathbf{K}, \mathbf{V}, \mathbf{O}\}$. 

We first define
\begin{equation}
\mathbf{KV}_0 =\mathbf 0\in \mathbb{R}^{d \times d}, \\
\mathbf{KV}_t = \sum_{s \le tB} \lambda^{tB-s} \mathbf{k}_s \mathbf{v}_s^{\top}.
\end{equation}
Given $\mathbf{KV}_t$, the output of $(t+1)$-th block, i.e., $tB+r$, with $1 \le r \le B$ is
\begin{equation}
\begin{aligned}
&\mathbf{o}_{tB+r}^{\top} \\
= &\mathbf{q}_{tB+r}^{\top} \sum_{s \le tB+r} \lambda^{tB+r-s} \mathbf{k}_s \mathbf{v}_s^{\top} \\
= &\mathbf{q}_{tB+r}^{\top}\left( \sum_{s=tB+1}^{tB+r} \lambda^{tB+r-s} \mathbf{k}_s \mathbf{v}_s^{\top} + \lambda^r \sum_{s \le tB} \lambda^{tB-s} \mathbf{k}_s \mathbf{v}_s^{\top} \right) \\
= & \mathbf{q}_{tB+r}^{\top} \sum_{s=tB+1}^{tB+r} \lambda^{tB+r-s} \mathbf{k}_s \mathbf{v}_s^{\top} + \lambda^r \mathbf{q}_{tB+r} \mathbf{kv}_{tB}^{\top} .
\end{aligned}
\end{equation}
Rewritten in matrix form, we have
\begin{equation}
\begin{aligned}
\mathbf{O}_{t+1}=  &
\underbrace{[(\mathbf{Q}_{t+1} \mathbf{K}_{t+1}^{\top}) \odot \mathbf{M}] \mathbf{V}_{t+1}}_{\mathrm{Intra\ Block}} \\
&+ \underbrace{\Lambda\mathbf{Q}_{t+1} (\mathbf{KV}_t)}_{\mathrm{Inter\ Block}},\\
\end{aligned}
\end{equation}
where
\begin{equation}
\begin{aligned}
\mathbf{M}_{ts} &= \begin{cases}
\lambda^{t-s} & t \ge s\\
0 &t < s
\end{cases} , \\
\Lambda&=\mathrm{diag}\{1, \ldots, \lambda^{B-1}\} .
\end{aligned}
\end{equation}
And the $\mathbf{KV}$ at $(t+1)$-th block can be written as
\begin{equation}
\begin{aligned}
\mathbf{KV}_{t+1} &= \sum_{s \le (t+1)B} \lambda^{(t+1)B-s} \mathbf{k}_s^{\top} \mathbf{v}_s \\
&= \lambda^B \sum_{s \le tB} \lambda^{tB-s} \mathbf{k}_s^{\top} \mathbf{v}_s + \sum_{s=tB+1}^{(t+1)B} \lambda^{(t+1)B-s} \mathbf{k}_s^{\top} \mathbf{v}_s \\
&= \lambda^B \mathbf{KV}_t + \left(\mathrm{diag}\{\lambda^{B-1}, \ldots, 1\} \mathbf{K}_{t}\right)^{\top} \mathbf{V}_{t} \\
&= \lambda^B \mathbf{KV}_t + \left(\lambda^{B}\Lambda^{-1} \mathbf{K}_{t}\right)^{\top} \mathbf{V}_{t}.
\end{aligned}
\end{equation}
The complete expression of the forward pass of Lightning Attention with decay can be found in Algorithm~\ref{algo:Lightning Attention with decay fw pseudo}.

\subsubsection{Backward Pass}
For backward pass, let us consider the reverse process.
First given $\mathbf{do}_t$, we have
\begin{equation}
\begin{aligned}
\mathbf{dq}_t^{\top} &= \mathbf{do}_t^{\top} \mathbf{kv}_t^\top \in \mathbb{R}^{1 \times d}, \\
\mathbf{dk}_t^{\top} &= \mathbf{v}_t ^{\top}\mathbf{dkv}_t^\top \in \mathbb{R}^{1 \times d}, \\
\mathbf{dv}_t^{\top} &= \mathbf{k}_t^{\top} \mathbf{dkv}_t \in \mathbb{R}^{1 \times d}, \\
\mathbf{dkv}_t &= \sum_{s \geq t} \lambda^{s-t} \mathbf{q}_s \mathbf{do}_s^{\top} \in \mathbb{R}^{d \times d}.
\end{aligned}
\end{equation}
By writing $\mathbf{dkv}_t$ in a recursive form, we get
\begin{equation}
\begin{aligned}
\mathbf{dkv}_{n+1} &= 0 \in \mathbb{R}^{d \times d}, \\\quad \mathbf{dkv}_{t-1} &= \lambda \mathbf{dkv}_t + \mathbf{q}_{t-1} \mathbf{do}_{t-1}^\top.
\end{aligned}
\end{equation}
To facilitate the understanding of tiling, let us consider the above equations in block style. Given the total sequence length $n$ and block size $B$, $\mathbf{X}$ is divided into $T = \frac{n}{B}$ blocks $\{\mathbf{X}_1, \mathbf{X}_2, \ldots, \mathbf{X}_T\}$ of size $B \times d$ each, where $\mathbf{X} \in \{ \mathbf{Q}, \mathbf{K}, \mathbf{V}, \mathbf{O}, \mathbf{dO} \}$. 

We first define
\begin{equation}
    \begin{aligned} 
    \mathbf{dKV}_{T+1}&=\mathbf 0 \in \mathbb{R}^{d \times d}, \\
    \mathbf{dKV}_t &= \sum_{s > tB} \lambda^{s-tB} \mathbf{q}_s \mathbf{do}_s^\top.
    \end{aligned}
\end{equation}

Then for the $(t+1)$-th block, i.e., $tB+r, 0 \leq r < B$, we have
\begin{equation}
\begin{aligned}
&\mathbf{dq}_{tB+r}^\top \\
=& \mathbf{do}_{tB+r}^\top \sum_{s \leq tB+r} \lambda^{tB+r-s} \mathbf{v}_s \mathbf{k}_s^\top \\
=& \mathbf{do}_{tB+r}^\top \left(\sum_{s=tB+1}^{tB+r} \lambda^{tB+r-s} \mathbf{v}_s \mathbf{k}_s ^\top
+ \lambda^r \sum_{s \leq tB} \lambda^{tB-s} \mathbf{v}_s\mathbf{k}_s ^\top 
\right) \\
=&\mathbf{do}_{tB+r}^\top \sum_{s=tB+1}^{tB+r} \lambda^{tB+r-s} \mathbf{v}_s \mathbf{k}_s^\top +  \lambda^r \mathbf{do}_{tB+r} \mathbf{kv}_{tB}^\top .
\end{aligned}
\end{equation}
In matrix form, we have
\begin{equation}
\begin{aligned}
\mathbf{dQ}_{t+1} =&
\underbrace{[(\mathbf{dO}_{t+1} \mathbf{V}_{t+1}^\top) \odot \mathbf{M}] \mathbf{K}_{t+1}}_{{\mathrm{Intra\ Block}}} \\
&+ \underbrace{\Lambda \mathbf{dO}_{t+1} (\mathbf{KV}_t^\top)}_{{\mathrm{Inter\ Block}}}.
\end{aligned}
\end{equation}
Since the recursion of $\mathbf{dK}_t$ steps from $t+1$ to $t$, given $\mathbf{KV}_{t+1}$, $\mathbf{dK}_{t}$ for the $t$-th block, i.e., at positions $(t-1)B+r, 0< r \le B$ is
\begin{equation}
\begin{aligned}
&\mathbf{dk}_{(t-1)B+r}^\top \\
= &\mathbf{v}_{(t-1)B+r}^\top \sum_{s \geq (t-1)B+r} \lambda^{s-(t-1)B-r} \mathbf{do}_s \mathbf{q}_s^\top \\
=& \mathbf{v}_{(t-1)B+r}^\top \left(
\sum_{s=(t-1)B+r}^{tB} \lambda^{tB+r-s} \mathbf{do}_s\mathbf{q}_s^\top  \right) \\
&+
 \mathbf{v}_{(t-1)B+r}^\top\left(\lambda^{B-r} \sum_{s > tB} \lambda^{s-tB} \mathbf{do}_s \mathbf{q}_s^\top \right)
\\
= &\mathbf{v}_{(t-1)B+r}^\top \sum_{s=(t-1)B+r}^{tB} \lambda^{tB+r-s} \mathbf{do}_s \mathbf{q}_s^\top \\
&+ \lambda^{B-r} \mathbf{v}_{(t-1)B+r} ^\top\mathbf{dKV}_{t}^\top.
\end{aligned}
\end{equation}
In matrix form, we get
\begin{equation}
\begin{aligned}
\mathbf{dK}_{t-1} =&
\underbrace{[(\mathbf{dO}_{t-1} \mathbf{V}_{t-1}^\top) \odot \mathbf{M}]^\top \mathbf{Q}_{t-1}}_{{\mathrm{Intra\ Block}}} \\
&+ \underbrace{
\lambda^{B} \Lambda^{-1}
\mathbf{V}_{t-1} (\mathbf{dKV}_t^\top)}_{{\mathrm{Inter\ Block}}}.
\end{aligned}
\end{equation}

Considering $\mathbf{dV}_t$ for the $t$-th block, i.e., at positions $(t-1)B+r, 0 < r \le B$, we have
\begin{equation}
\begin{aligned}
&\mathbf{dv}_{(t-1)B+r}^\top \\
=& \mathbf{k}_{(t-1)B+r}^\top \sum_{s \geq (t-1)B+r} \lambda^{s-(t-1)B-r} \mathbf{q}_s\mathbf{do}_s^\top  \\
=& \mathbf{k}_{(t-1)B+r}^\top \left(
\sum_{s=(t-1)B+r}^{tB} \lambda^{tB+r-s} \mathbf{q}_s^\top \mathbf{do}_s \right) \\
&+\mathbf{k}_{(t-1)B+r}^\top \left( \lambda^{B-r} \sum_{s> tB} \lambda^{s-tB} \mathbf{q}_s \mathbf{do}_s^\top \right) \\
=& \mathbf{k}_{(t-1)B+r}^\top \sum_{s=(t-1)B+r}^{tB} \lambda^{tB+r-s} \mathbf{q}_s \mathbf{do}_s^\top \\
&+ \lambda^{B-r} \mathbf{k}_{(t-1)B+r} ^\top\mathbf{dKV}_{t} .
\end{aligned}
\end{equation}
In matrix form, we get
\begin{equation}
\begin{aligned}
\mathbf{dV}_{t-1} =&
\underbrace{[(\mathbf{Q}_{t-1} \mathbf{K}_{t-1}^\top) \odot \mathbf{M}]^\top \mathbf{dO}_{t}}_{{\mathrm{Intra\ Block}}} \\
&+ \underbrace{\lambda^{B} \Lambda^{-1} \mathbf{K}_{t-1} (\mathbf{dKV}_t)}_{{\mathrm{Inter\ Block}}}.
\end{aligned}
\end{equation}
Finally, the recursive relation for $\mathbf{dKV}_t$ is 
\begin{equation}
\begin{aligned}
\mathbf{dKV}_{t}
&= \sum_{s > tB} \lambda^{s-tB} \mathbf{q}_s\mathbf{do}_s^\top  \\
&= \lambda^B \sum_{s > (t+1)B} \lambda^{s-(t+1)B} \mathbf{q}_s\mathbf{do}_s^\top  \\
&+ \sum_{s=tB+1}^{(t+1)B} \lambda^{s-tB} \mathbf{q}_s\mathbf{do}_s^\top  \\
&= \lambda^B \mathbf{dKV}_{t+1} + \left(\Lambda \mathbf{Q}_t \right)^\top \mathbf{dO}_t.
\end{aligned}
\end{equation}
Algorithm~\ref{algo:Lightning Attention with decay bw pseudo} describes the backward pass of Lightning Attention with decay in more detail.


\section{Corpus}
\label{app:corpus}
We gather an extensive corpus of publicly accessible text from the internet, totaling over $700$TB in size. The collected data are processed by our data preprocessing procedure as shown in Fig.~\ref{fig:data_preprocess}, leaving a $6$TB cleaned corpus with roughly 2 trillion tokens. We categorize our data sources to provide better transparency and understanding. The specifics of these categories are outlined in Table~\ref{tab:pretraing_data}. 

\subsection{Data Preprocessing}
\begin{figure*}[htbp]
    \centering
    \vspace{-0mm}
        \includegraphics[width=0.95\textwidth]{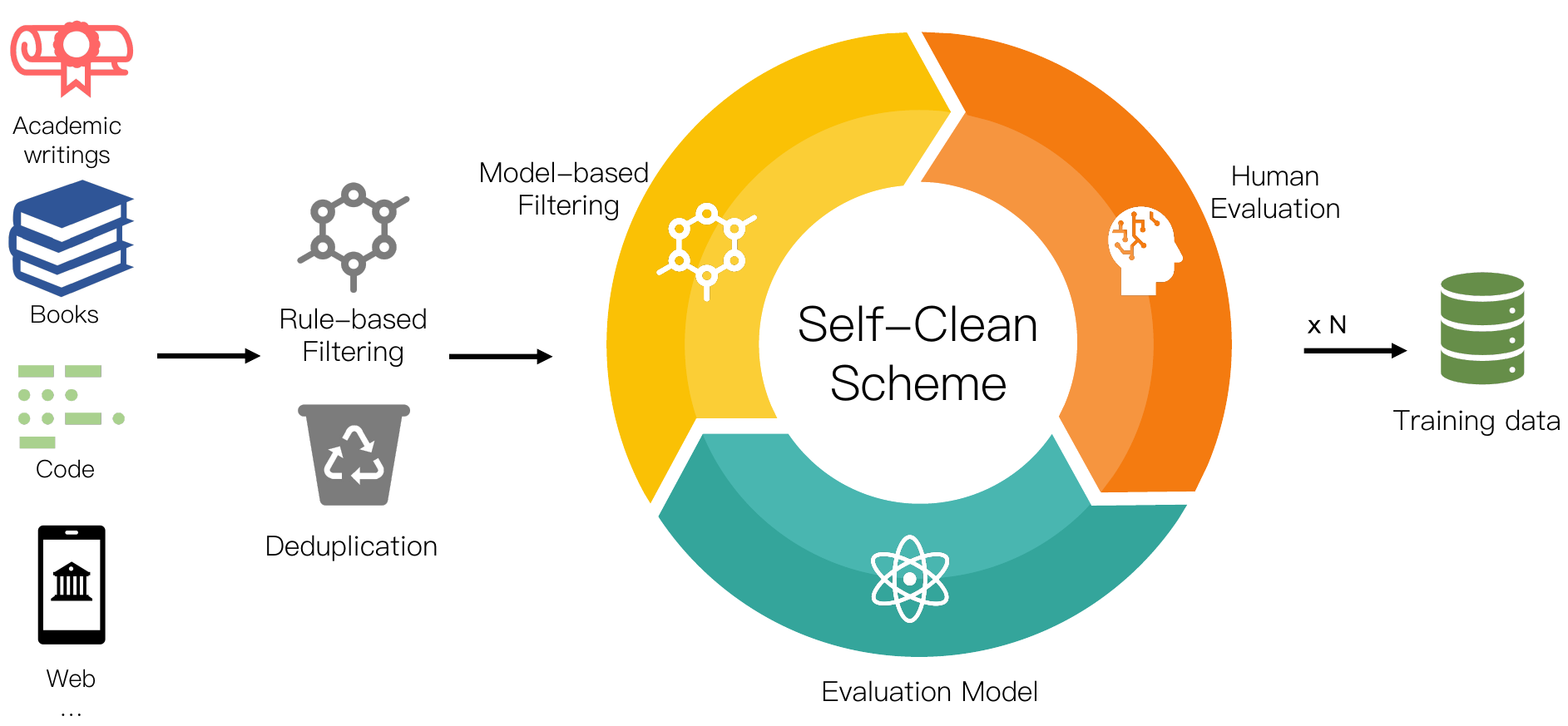}
        \vspace{-4mm}
        \caption{\textbf{Data Preprocess Procedure.} The collected data undergoes a process of rule-based filtering and deduplication, followed by our self-clean data processing strategy: model-based filtering, human evaluation, and evaluation model. After several iterations of the above cycle, we obtain high-quality training data at around 2T tokens.}
        \label{fig:data_preprocess}
\end{figure*}

Our data preprocessing procedure consists of three steps: 1). rule-based filtering, 2). deduplication, and 3). a self-cleaning scheme. Before being added to the training corpus, the cleaned corpus needs to be evaluated by humans.

\paragraph{Rule-based filtering}
The rules we used to filter our collected data are listed as follows:
\begin{itemize}
    \item \emph{Removal of HTML Tags and URLs:} The initial step in our process is the elimination of HTML tags and web URLs from the text. This is achieved through regular expression techniques that identify these patterns and remove them, ensuring the language model focuses on meaningful textual content.
    \item \emph{Elimination of Useless or Abnormal Strings:} Subsequently, the cleaned dataset undergoes a second layer of refinement where strings that do not provide value, such as aberrant strings or garbled text, are identified and excised. This process relies on predefined rules that categorize certain string patterns as non-contributing elements.
    \item \emph{Deduplication of Punctuation Marks:} We address the problem of redundant punctuation marks in the data. Multiple consecutive punctuation marks can distort the natural flow and structure of sentences when training the model. We employ a rule-based system that trims these duplications down to a single instance of each punctuation mark.
    \item \emph{Handling Special Characters:} Unusual or special characters that are not commonly part of the language's text corpus are identified and either removed or replaced with a standardized representation.
    \item \emph{Number Standardization:} Numerical figures may be presented in various formats across different texts. These numbers are standardized into a common format to maintain consistency.
    \item \emph{Preservation of Markdown/LaTeX Formats:} While removing non-textual elements, exceptions are made for texts in Markdown and LaTeX formats. Given their structured nature and ubiquitous use in academia and documentation, preserving these formats can enhance the model's ability to understand and generate similarly formatted text.
\end{itemize}

\paragraph{Deduplication}
To ensure the uniqueness of our data and avert the risk of overfitting, we employ an efficient de-duplication strategy at the document or line level using MinHash and Locality-Sensitive Hashing (LSH) algorithms. This combination of MinHash and LSH ensures a balance between computational efficiency and accuracy in the deduplication process, providing a robust mechanism for data deduplication and text watermark removal. 

\paragraph{Self-cleaning scheme}
Our data self-cleaning process involves an iterative loop of the following three steps to continuously refine and enhance the quality of our dataset. An issue of using model-based data filters is that the filtered data will have a similar distribution as the evaluation model, which may have a significant impact on the diversity of the training data. Assuming that the majority of the pre-processed data is of high quality, we can train an evaluation model on the entire set of pre-processed data, and the model will automatically smooth the data manifold distribution and outlet low-quality data while retaining the majority of the diversities. 

The self-cleaning scheme unfolds as follows:
\begin{itemize}
    \item \emph{Evaluation Model:} We train a 385M model on the pre-processed corpus to act as a data quality filter. 
    \item \emph{Model-Based Data Filtering:} We use the evaluation model to assess each piece of data with perplexity. Only data achieving a score above a certain threshold is preserved for the next step. Low-quality data are weeded out at this stage.
    \item \emph{Human Evaluation:} We sample a small portion of the filtered data and manually evaluate the quality. 
\end{itemize}

These steps are repeated in cycles, with each iteration improving the overall quality of the data and ensuring the resulting model is trained on relevant, high-quality text. This self-cleaning process provides a robust mechanism for maintaining data integrity, thereby enhancing the performance of the resulting language model.

\begin{table}[t]
\small
    \caption{\textbf{Statistics of our corpus.} For each category, we list the number of epochs performed on the subset when training on the 2 trillion tokens, as well as the number of tokens and disk sizes. We also list the table on the right according to the language distribution.  }
    \label{tab:pretraing_data}
    \centering
    \begin{minipage}{0.49\textwidth}
        \centering
        \setlength{\tabcolsep}{2mm}
        \begin{tabular}{p{3.5cm}crr}
        \toprule[0.8pt]
        Dataset & Epochs & Tokens & Disk size \\ \hline
        Academic Writings & 1.53 & 200 B & 672 GB \\ 
        Books & 2.49 & 198 B & 723 GB \\
        Code & 0.44 & 689 B & 1.4 TB \\
        Encyclopedia & 1.51 & 5 B & 18 GB \\
        Filtered Webpages & 1.00 & 882 B & 3.1 TB \\
        Others & 0.63 & 52 B & 154 GB \\ \hline
        Total & - & 2026 B  & 6 TB \\ \toprule[0.8pt]
        \end{tabular}
    \end{minipage}%
    \hfill
    \begin{minipage}{0.5\textwidth}
        \centering
        \setlength{\tabcolsep}{8mm}
        \begin{tabular}{cccr}
        \toprule[0.8pt]
        \multicolumn{2}{c}{Language} & Tokens & Disk size \\ \hline
        \multicolumn{2}{c}{English} & 743 B & 2.9 TB \\
        \multicolumn{2}{c}{Chinese} & 555 B & 1.7 TB \\
        \multicolumn{2}{c}{Code} & 689 B & 1.4 TB \\
        \multicolumn{2}{c}{Others} & 39 B & 89 GB \\ \hline
        \multicolumn{2}{c}{Total} & 2026 B & 6 TB \\ \toprule[0.8pt]
        \end{tabular}
    \end{minipage}%
\end{table}


\subsection{Tokenization}
We tokenize the data with the Byte-Pair Encoding (BPE) algorithm. Notably, to enhance compatibility with Chinese language content, a significant number of common and uncommon Chinese characters have been incorporated into our vocabulary. In cases where vocabulary items are not present in the dictionary, the words are broken down into their constituent UTF-8 characters. This strategy ensures comprehensive coverage and flexibility for diverse linguistic input during model training.

\section{Distributed System Optimization}
We optimize our system to execute large-scale pre-training for TNL effectively. We employ fully sharded data parallelism (FSDP)~\cite{zhao2023pytorch}, activation checkpointing~\cite{shoeybi2019megatron}, and automatic mixed precision (AMP)~\cite{micikevicius2017mixed} techniques to reduce memory footprint and expedite computational speed. We used BFloat16~\cite{kalamkar2019study} to enhance training stability. We implemented model parallelism tailored to Lightning Attention. Inspired by Megatron-LM~\cite{shoeybi2019megatron} model parallelism, which independently addresses self-attention and MLP blocks, we apply model parallelism to SGLU and GLA separately. The details of our model parallelism strategies are elaborated below.

\textbf{SGLU Model Parallelism}
Recall SGLU structure in (\ref{eq: glu}):
\begin{equation}
\small
\mathbf O=[(\mathbf X \mathbf W_v) \odot (\mathbf X \mathbf W_u)]\mathbf W_o,
\label{eq: mp_glu}
\end{equation}
The model parallelism adaptation of SGLU is as follows:
\begin{equation}
\small
\begin{aligned}
[\mathbf {O}'_1, \mathbf {O}'_2]&=\mathbf X[\mathbf W_v^1, \mathbf W_v^2] \odot \mathbf X[\mathbf W_u^1, \mathbf W_u^2]\\
&=[\mathbf X \mathbf W_v^1, \mathbf X \mathbf W_v^2] \odot [\mathbf X \mathbf W_u^1, \mathbf X \mathbf W_u^2],
\label{eq: mp_glu_1}
\end{aligned}
\end{equation}
\normalsize
which splits the weight matrices $\mathbf W_v$ and $\mathbf W_u$ along their columns and obtains an output matrix splitting along its columns too.
Then the split output $[\mathbf O_1, \mathbf O_2]$ is multiplied by another matrix which is split along its rows as:
\begin{equation}
\small
\mathbf {O}=[\mathbf O_1', \mathbf O_2'] [\mathbf W_o^1, \mathbf W_o^2]^\top=\mathbf O_1' \mathbf W_o^1 + \mathbf O_2' \mathbf W_o^2
\label{eq: mp_glu_output}
\end{equation}
Similar to model parallelism in Megatron-LM, this whole procedure splits three general matrix multiplies (GEMMs) inside the SGLU block across multiple GPUs and only introduces a single \textit{all-reduce} collective communication operation in both the forward and backward passes, respectively. 

\textbf{GLA Model Parallelism}
Recall the GLA block in (\ref{eq: gla}), its model parallelism version is:
\begin{equation}
\small
[\mathbf{O_1}, \mathbf{O_2}]=\mathrm{SRMSNorm}(\mathbf{Q} \mathbf{K}^{\top}\mathbf{V})\odot \mathbf{U},
\label{eq: mp_gla1}
\end{equation}
where:
\begin{equation}
\small
\begin{aligned}
\mathbf Q&=[\phi(\mathbf X \mathbf W_q^1), \phi(\mathbf X \mathbf W_q^2)],
\mathbf K=[\phi(\mathbf X \mathbf W_q^1), \phi(\mathbf X \mathbf W_q^2)],\\
\mathbf V&=\mathbf X [\mathbf W_v^1, \mathbf W_v^2], \mathbf U=\mathbf X [\mathbf W_u^1, \mathbf W_u^2],
\label{eq: mp_gla2}
\end{aligned}
\end{equation}
\normalsize
Note that in our implementation, we use the combined QKVU projection to improve computation efficiency for linear attention.
The obtained split output matrix $[\mathbf{O_1}, \mathbf{O_2}]$ again is multiplied by a weight matrix split along its columns which is similar to (\ref{eq: mp_glu_output}).

\begin{figure*}[t!]
    \centering
    \includegraphics[width=0.95\textwidth]{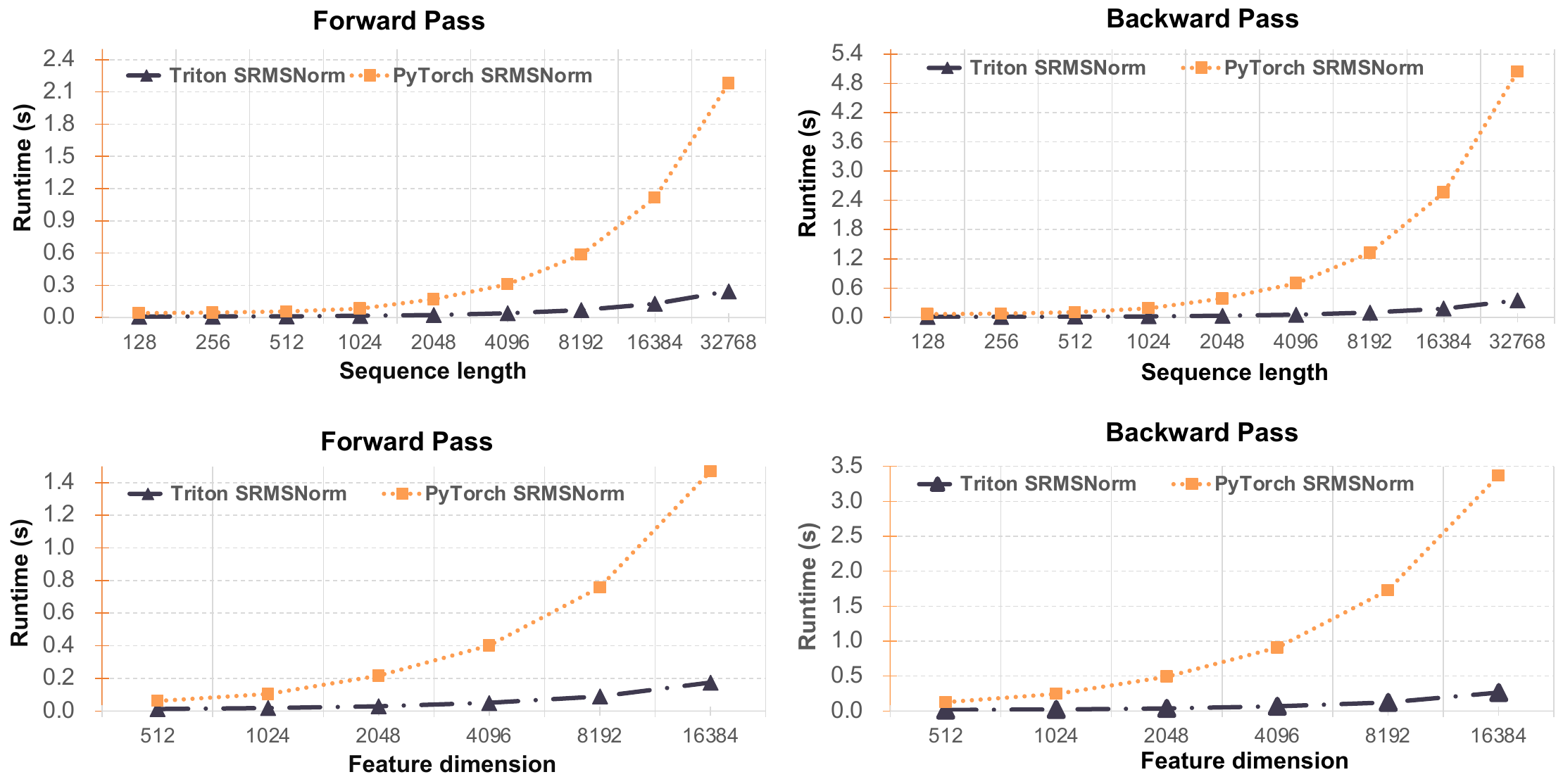}
    \vspace{-2mm}
    \caption{\textbf{Performance Evaluation of SRMSNorm Implementation.} The upper figures exhibit the runtime comparison of the forward pass (left section) and backward pass (right section) for different sequence lengths, with a fixed feature dimension of 3072. The lower two figures illustrate the runtime comparison for various feature dimensions, with a fixed sequence length of 4096.}
    \vspace{-4mm}
    \label{fig:norm}
\end{figure*}

\section{Additional TNL Ablation}

\paragraph{Transformer \emph{vs} TNL} 
We carried out a meticulous series of comparative tests between our TNL and Transformer, spanning over an array of disparate sizes. The comparative performance of these models is clearly illustrated in Table~\ref{tab:transf_vs_transn}. 
Under identical configurations, it becomes evident that our TNL exhibits a superior performance profile compared to Transformer. We observed that TNL outperformed Transformer by a remarkable 5\% at the size of 385M. More importantly, as the size reached 1B, this superiority became even more pronounced, with an advantage of 9\% for TNL over Transformer.

\begin{table}[h]
\centering
    \small
    \vspace{-4mm} 
    \caption{\small\textbf{Transformer \emph{vs} TNL.} TNL performs better than Transformer in size of 385M and 1B under identical configurations by 5\% and 9\%, respectively.}
    \vspace{1mm} 
    \label{tab:transf_vs_transn}
    \setlength{\tabcolsep}{4.0mm}
    \begin{tabular}{ccccc}
    \toprule
    \small
    \\[-1em]
    \\[-1em]
    \multicolumn{2}{c}{Method}      & \multicolumn{1}{c}{Updates} & Loss & PPL  \\ \hline
    \\[-1em]
    \multicolumn{2}{l}{Transformer-385M}    & 100K  & 2.362 &5.160   \\
    \multicolumn{2}{l}{TNL-385M} & 100K  & 2.248 &4.770 \\ \hline
    \multicolumn{2}{l}{Transformer-1B}      & 100K  & 2.061 &4.765 \\ 
    \multicolumn{2}{l}{TNL-1B}   & 100K  & 1.896 &3.729 \\

    \bottomrule
\end{tabular}
\vspace{-3mm}
\end{table}

\begin{table}[h]
    \centering
  \small
  \vspace{-5mm}
    \caption{\textbf{TransNormer \emph{vs} TNL.}  TNL performs better than TransNormer.}
 \vspace{1mm} 
    \label{tab:transnormers}
    \centering
    \setlength{\tabcolsep}{2.6mm}
        \begin{tabular}{lllll}
        \toprule
        \\[-1em]
        Method & Params & Updates & Loss & PPL \\ \hline
        \\[-1em]
        TNL & 385M & 100K  & 2.248 &4.770 \\ 
        TransNormer-T1 & 379M & 100K & 2.290 &4.910 \\ 
        TransNormer-T2 & 379M & 100K & 2.274 &4.858\\ 
        \bottomrule
    \end{tabular}
    \vspace{-3mm}
\end{table}

We compare the original TransNormer and the improved TNL and the results are shown in Table~\ref{tab:transnormers}. TNL exhibited an enhancement of 2\% and 1\% respectively.

\paragraph{Speed Normalization Fucntions} We enhanced SRMSNorm using Triton, resulting in notable improvements in processing speed for larger dimensions, as shown in Fig. ~\ref{fig:norm}, outperforming conventional PyTorch implementations.

\end{document}

%% file: math_commands.tex
\usepackage{amsmath,amsfonts,bm}

\def\eqref#1{equation~\ref{#1}}

\def\1{\bm{1}}

\DeclareMathAlphabet{\mathsfit}{\encodingdefault}{\sfdefault}{m}{sl}
\SetMathAlphabet{\mathsfit}{bold}{\encodingdefault}{\sfdefault}{bx}{n}

\newcommand{\R}{\mathbb{R}}

